\documentclass{article}

\usepackage[final,nonatbib]{neurips_2023}

\usepackage[utf8]{inputenc} 
\usepackage[T1]{fontenc}    
\usepackage{hyperref}       
\usepackage{url}            
\usepackage{booktabs}       
\usepackage{amsfonts}       
\usepackage{nicefrac}       
\usepackage{microtype}      
\usepackage{xcolor}         

\usepackage{amsmath,amsfonts}
\usepackage{array}
\usepackage[caption=false]{subfig}
\usepackage{textcomp}
\usepackage{stfloats}
\usepackage{url}
\usepackage{verbatim}
\usepackage{graphicx}
\usepackage{subfig}
\usepackage{pifont}
\usepackage{amsmath}
\usepackage{amssymb}
\usepackage{booktabs}
\usepackage{url}
\usepackage{multirow}
\usepackage{diagbox}
\usepackage{bbding}
\usepackage{amsthm}
\usepackage{amsfonts}
\usepackage{float}  
\usepackage{color}  
\usepackage{textcomp}
\usepackage{array}  
\usepackage{wrapfig}
\usepackage{cleveref}
\usepackage[ruled]{algorithm2e}

\newtheorem{definition}{Definition}
\newtheorem{lemma}{Lemma}
\usepackage{wrapfig}

\newcommand{\ie}{\textit{i}.\textit{e}., }
\newcommand{\eg}{\textit{e}.\textit{g}., }

\newcommand{\aka}{\textit{a}.\textit{k}.\textit{a} }

\title{Cross-modal Active Complementary Learning \\ with Self-refining Correspondence}

\author{
    Yang Qin\textsuperscript{\rm 1},\ \ Yuan Sun\textsuperscript{\rm 1},\ \ Dezhong Peng\textsuperscript{\rm 1,3,4},\ \ Joey Tianyi Zhou\textsuperscript{\rm 2},\\\textbf{Xi Peng\textsuperscript{\rm 1},\ \ \  \ \ \   Peng Hu\textsuperscript{\rm 1}\thanks{Corresponding author.}}\\
     \textsuperscript{\rm 1} College of Computer Science, Sichuan University, Chengdu, China.\\
    \textsuperscript{\rm 2} Centre for Frontier AI Research (CFAR) and Institute of High \\Performance Computing (IHPC), A*STAR, Singapore.\\
    \textsuperscript{\rm 3} Chengdu Ruibei Yingte Information Technology Co., Ltd, Chengdu, China.\\
    \textsuperscript{\rm 4} Sichuan Zhiqian Technology Co., Ltd, Chengdu, China.\\
    \texttt{\{qinyang.gm,joey.tianyi.zhou,pengx.gm,penghu.ml\}@gmail.com, }\\
    \texttt{sunyuan\_work@163.com, pengdz@scu.edu.cn}
}

\begin{document}

\maketitle

\begin{abstract}

Recently, image-text matching has attracted more and more attention from academia and industry, which is fundamental to understanding the latent correspondence across visual and textual modalities. However, most existing methods implicitly assume the training pairs are well-aligned while ignoring the ubiquitous annotation noise, \aka noisy correspondence (NC), thereby inevitably leading to a performance drop. Although some methods attempt to address such noise, they still face two challenging problems: excessive memorizing/overfitting and unreliable correction for NC, especially under high noise. To address the two problems, we propose a generalized Cross-modal Robust Complementary Learning framework (CRCL), which benefits from a novel Active Complementary Loss (ACL) and an efficient Self-refining Correspondence Correction (SCC) to improve the robustness of existing methods.   Specifically, ACL exploits active and complementary learning losses to reduce the risk of providing erroneous supervision, leading to theoretically and experimentally demonstrated robustness against NC. SCC utilizes multiple self-refining processes with momentum correction to enlarge the receptive field for correcting correspondences, thereby alleviating error accumulation and achieving accurate and stable corrections. We carry out extensive experiments on three image-text benchmarks, \ie Flickr30K, MS-COCO, and CC152K, to verify the superior robustness of our CRCL against synthetic and real-world noisy correspondences. Code is available at \url{https://github.com/QinYang79/CRCL}.

\end{abstract}
\section{Introduction}

Image-text matching aims to search the most relevant samples across different modalities, which is fundamental for most cross-modal tasks \cite{malinowski2015ask,vinyals2015show,xu2015show,diao2021similarity,jiang2023cross,9782584}. 
The core of image-text matching is how to accurately measure the similarity between distinct modalities, however, which is challenging due to the visual-textual discrepancy. To tackle the challenge, numerous deep methods are presented to learn the visual-semantic associations of image-text pairs and achieve remarkable progress, thanks to the powerful representation ability of Deep Neural Networks (DNNs) and some well-designed similarity inference architectures~\cite{faghri2017vse++,lee2018stacked,diao2021similarity,chen2021learning}. They could be roughly divided into two groups, \ie global-level methods \cite{faghri2017vse++,chen2021learning} and local-level methods \cite{lee2018stacked,diao2021similarity}, which aim at learning the image-to-sentence and region-to-word correlation to infer the cross-modal similarity, respectively. 
Although these methods achieved promising matching performance, most of them implicitly require large-scale well-aligned data for training, which is expensive or even impossible to collect due to ubiquitous noise in real-world scenarios \cite{sharma2018conceptual,jia2021scaling}.
Therefore, there is inevitably imperfect alignment luring in the data, \ie noisy correspondence (NC)~\cite{huang2021learning}, resulting in inferior performance.

To handle the NC problem, some prior arts are presented to alleviate the adverse impact of NC in various tasks, \eg partially view-aligned clustering~\cite{yang2022robust,yang2021partially,wen2023deep,yang2020adversarial}, video-text retrieval~\cite{zhang2023robust}, visible-infrared person re-identification~\cite{yang2022learning}, and image-text matching~\cite{huang2021learning,qin2022deep,hu2023cross}. Specifically, inspired by learning with noisy labels~\cite{li2020dividemix,han2018co}, some works~\cite{huang2021learning,yang2023bicro,han2023noisy} are proposed to alleviate the negative impact brought by NC. These works attempt to leverage the memorization effect of DNNs~\cite{arpit2017closer} to gradually distinguish the noisy image-text pairs for robust learning in a co-teaching manner. Furthermore, the predicted soft correspondence is used to recast a soft margin to replace the scalar margin of triplet ranking loss~\cite{faghri2017vse++}, which helps to avoid misleading the model by mismatched pairs.
However, the soft-margin ranking loss is experimentally found to provide only limited robustness against NC, especially under high noise (as shown in \Cref{tb1}), due to unstable division based on inaccurate predictions.
In contrast, some works~\cite{qin2022deep,hu2023cross} aim to enhance the robustness of cross-modal methods against NC by starting with a robust loss function, \ie avoiding over-amplification of wrong supervision information to reduce misleading risk. However, the lack of explicitly mitigating the effect of easily separable noise makes them hard to further improve the performance.

To address the aforementioned issues, we propose a generalized robust framework, dubbed CRCL, for learning with noisy correspondences. CRCL could be easily migrated into existing image-text matching methods to enhance their robustness against NC. Our framework introduces a novel Active Complementary Loss (ACL) that applies active and complementary learning to mutually boost robustness against NC. Specifically, we present a robust complementary learning loss that employs complementary pairs, such as ``input image-text pairs are irrelevant'', to conduct indirect cross-modal learning with exponential normalization.
Due to the low likelihood of selecting incorrect complementary pairs, robust complementary learning could reduce the risk of providing incorrect supervision and smooth the losses, thus embracing robustness against NC. However, the robust loss will face the underfitting problem, leading to suboptimal performance. To overcome this issue, a weighted Active learning loss is proposed to enforce the model focus on more reliable positive pairs in addition to only complementary pairs. In addition, we propose a Self-refining Correspondence Correction paradigm (SCC) to obtain stable and accurate correspondence correction. SCC utilizes Momentum Correction (MC) to aggregate historical predictions for stable and accurate correspondence corrections. By combining multiple Self-Refining processes (SR) throughout the entire training process, we alleviate over-memorization for NCs. In summary, the key contributions and innovations of this work are as follows:

\begin{itemize}
    \item We propose a generalized Cross-modal Robust Contrastive Learning framework (CRCL) to address a pressing and widely-exist problem in image-text matching, \ie noisy correspondence. CRCL empowers existing methods with strong robustness through the perspectives of robust loss and correction techniques.
    \item A novel Active Complementary Loss (ACL) is presented to balance active and complementary learning, mutually enhancing robustness against NC while encapsulating correct cross-modal associations in the latent common space.
    \item We design an effective Self-refining Correspondence Correction paradigm (SCC) to achieve accurate and stable soft correspondences, which enables the prediction-based corrections to perceive larger fields and self-refine from the historically learned knowledge.
    \item Extensive experiments verify the effectiveness and superiority of our framework on three benchmark image-text datasets: Flickr30K, MS-COCO, and CC152K. Additionally, comprehensive ablation studies and insightful analyses demonstrate the reliability and practicability of the proposed CRCL.
\end{itemize}

\section{The Proposed Method}
\subsection{Preliminaries and Problem Statement}

To be specific, we first provide some definitions of instance-level image-text matching so as to conveniently study noisy correspondence. Let $\mathcal{D}=\{\mathcal{I},\mathcal{T},\mathcal{Y}\}$ be an image-text dataset, where $\mathcal{I}=\left\{ I_i \right\}^N_{i=1}$ and $\mathcal{T}=\left\{ T_i \right\}^N_{i=1}$ are the training image and text set with size of $N$. The correspondence label space is defined as $\mathcal{Y} = \left\{ y_{ij}|i=1,\cdots,N;j=1,\cdots,N \right\}$, where $y_{ij}$ represents the correspondence of pair $(I_i,T_j)$, \ie if $I_i$ and $T_j$ are matched (\ie positive pair), $y_{ij}=1$ otherwise $y_{ij}=0$. We assume each pair with the same indices has matched correspondence, \ie $y_{ii} = 1, i=1,\cdots, N$. However, due to the ubiquitous noise during data collection, some negative pairs are mismatched as positives, \aka noisy correspondence (NC)~\cite{huang2021learning,qin2022deep}, which would introduce wrong supervisory information and misguide model training, leading to performance degradation. Mathematically, we define NC as shown in~\Cref{def1}.

\begin{definition}
    Due to the existence of NC, the learner only has access to the noisy training data ($\mathcal{D}_\eta$), instead of clean data ($\mathcal{D}$). Thus, the correspondence label for pair $(I_i,T_j)$ is reconsidered as
\begin{equation}
    \tilde{y}_{ij} = \left\{ \begin{array}{ll}
         y_{ij}&  \text{with probability }(1-\eta_{ij} ), \\
          1 - y_{ik}& \text{with probability }  \bar{\eta}_{ik}, \forall k\neq j.\\
    \end{array}\right.
\end{equation}

For all pairs,  conditioned on that  if $i=j$ then $y_{ij} = 1$ else $y_{ij} = 0$, we have $\sum_{j \neq k} \bar{\eta}_{ik} =  {\eta}_{ij}$. Similar to the definitions of noisy labels~\cite{ghosh2017robust},  we assume that NC is uniform, \ie $\eta_{ij} = \eta$ and $\bar{\eta}_{ik} = \frac{\eta}{N-1}, \forall k\neq j$, where $\eta$ is a constant to represent the noise rate. 

\label{def1}
\end{definition}

The key to learning with noisy correspondence is to alleviate the misguiding impact and rectify the mismatched pairs. One direct solution is to enhance the robustness of loss function $\mathcal{L}$ against noisy pairs, which can help prevent overfitting on mismatched pairs. The second aims at using the memorization effect of DNNs~\cite{arpit2017closer} to discriminate the mismatched pairs, thus removing unreliable supervision from the training data.

For image-text matching, images and texts are first projected into a shared representation space by two modality-specific networks, denoted as $v$ and $g$, respectively. The cross-modal similarity of a given pair $(I_i,T_j)$ is then computed as $S\left(v(I_i), g(T_j)\right)$, where $S(*)$ could be the cosine function~\cite{faghri2017vse++,chen2021learning} or a relevance inference module~\cite{lee2018stacked,diao2021similarity}. For brevity, $S\left(v(I_i), g(T_j)\right)$ is denoted as $S\left(I_i, T_j\right)$ or $S_{ij}$ in the following. The computed similarities could be considered as supporting evidence for retrieved results. Thus, the learning objective of image-text matching is to maximize the cross-modal similarities of positive pairs while minimizing those of negatives in the latent common space, which is commonly achieved by using contrastive learning~\cite{faghri2017vse++,lee2018stacked,diao2021similarity}.

The widely-used triplet ranking loss~\cite{faghri2017vse++} has shown excellent performance in cross-modal contrastive learning tasks~\cite{faghri2017vse++,dong2019dual}. However, recent research~\cite{qin2022deep} has demonstrated that this loss function fails to perform well in image-text data with NCs, especially when using the hardest negative samples as comparison items. To address this issue, some works proposed an adaptive soft margin approach to improve the robustness of the ranking loss~\cite{huang2021learning,han2023noisy, yang2023bicro}, which is defined as follows:
\begin{equation}
    \mathcal{L}_{soft}\left(I_i, T_i\right)=\left[\hat{\alpha}_i-S\left(I_i, T_i\right)+S\left(I_i, \hat{T}_h\right)\right]_{+}+\left[\hat{\alpha}_i-S\left(I_i, T_i\right)+S\left(\hat{I}_h, T_i\right)\right]_{+},
\end{equation}
where $\hat{T}_h$ and $\hat{I}_h$ denote the hardest cross-modal samples in a mini-batch, $\hat{\alpha}_i$ is a soft margin adaptively computed by $\hat{\alpha}_i=\frac{m^{\hat{y}_{ii}}-1}{m-1} \alpha$, $\alpha$ is a constant margin, $m$ is a curve parameter, and $\hat{y}_{ii}$ is the rectified correspondence between $I_i$ and $T_i$. However, this approach has two disadvantages: 1) The margin setting $\alpha$ may not be consistent with the empirical setting under NC scenarios. 2) The inaccurate prediction $\hat{y}_{ii}$ can still easily produce the risk of misleading gradient, which can cause trivial solutions to fail, especially in the case of high noise (\eg the results of NCR~\cite{huang2021learning} and BiCro~\cite{yang2023bicro} on Flickr30K with 80\% noise). To overcome these limitations, we propose a novel Active Complementary Loss (ACL) under the risk minimization theory~\cite{manwani2013noise,ghosh2017robust} to provide noise tolerance for noisy pairs while ensuring discriminative learning.

\subsection{Active Complementary Loss}

For the image-text dataset $\mathcal{D}$, our goal is to learn a cross-modal model ($\mathcal{M}$) that can discriminatively identify the positive (matched) and negative (unmatched) pairs well for retrieval, which is intuitively equivalent to maximizing bidirectional matching probabilities of positives. The bidirectional matching probabilities for pair $(I_i, T_j)$ are defined as:

\begin{equation}
{p}^{\circ}_{ij} =
 f(I_i, T_j)
= \frac{e^{(S_{ij}/\tau)}}{\sum ^N_{l=1} e^{(S_{il}/\tau)}},\ \ \ 
{p}^{\diamond}_{ij} =
 f(T_j, I_i)
= \frac{e^{(S_{ij}/\tau)}}{\sum ^N_{l=1} e^{(S_{lj}/\tau)}},
\label{eq.3}
\end{equation}

where $\tau$ is a temperature parameter \cite{wu2018unsupervised,caron2020unsupervised,hu2021learning}, $f$ is regarded as the cross-modal decision function, and ``$\circ$, $\diamond$ '' means the two retrieval directions, \ie image-to-text and text-to-image, respectively. For any loss function $\mathcal{L}$, the matching risk of $f$ for image-text matching can be defined as
\begin{equation}
    R_\mathcal{L}(f) = \mathbb{E}_{(I_i,y_{i\cdot})\sim\mathcal{D}}\left[ \mathcal{L}(f(I_i,T_\cdot),y_{i\cdot}) \right] + \mathbb{E}_{(T_i,y_{\cdot i})\sim\mathcal{D}}\left[ \mathcal{L}(f(T_i,I_\cdot),y_{\cdot i}) \right],
\end{equation}
where $\mathbb{E}[\cdot]$ represents the expectation operator. Considering noisy correspondences, the risk of $f$ in noisy data $\mathcal{D}_\eta$ could be formulated as follows:
\begin{equation}
    R^\eta_\mathcal{L}(f) = \mathbb{E}_{(I_i,\tilde{y}_{i\cdot})\sim\mathcal{D}_\eta}\left[ \mathcal{L}(f(I_i,T_\cdot),\tilde{y}_{i\cdot}) \right] + \mathbb{E}_{(T_i,\tilde{y}_{\cdot i})\sim\mathcal{D}_\eta}\left[ \mathcal{L}(f(T_i,I_\cdot),\tilde{y}_{\cdot i}) \right],
\end{equation}
where $\tilde{y}_{ij}$ is the noisy correspondence label as shown in~\Cref{def1}. Thus, the cross-modal learning objective is to learn a model $\mathcal{M}^*_{\mathcal{L}}$ with a global minimizer $f^*_\eta$ of $R^\eta_\mathcal{L}(f)$. To achieve robustness, $f^*_\eta$ should be also the global minimizer of ${R}_\mathcal{L}(f)$ on the noise-free data.



 Inspired by complementary contrastive learning~\cite{hu2023cross}, we propose to optimize the matching probabilities of all negative pairs for learning with noisy data indirectly, thereby avoiding fast overfitting to NC. Simultaneously, to further improve the noise tolerance, we introduce an exponential normalization to smooth the complementary loss. Hence, the robust complementary loss for pair $(I_i,T_i)$ could be formulated as: 
\begin{equation}\label{rcl}
    \begin{aligned}
    \mathcal{L}_{r} (I_i,T_i,q) &=  \mathcal{L}^\circ_{r} (I_i,T_i,q) +  \mathcal{L}^\diamond_{r} (T_i,I_i,q) \\
    &=\sum^N\limits_{j \neq i}\tan  (p^\circ_{ij})/\Big(\sum^N\limits_{k = 1} \tan ( p^\circ_{ik} ) \Big)^q + \sum^N\limits_{j \neq i} \tan ( p^\diamond_{ji})/\Big(\sum^N\limits_{k = 1} \tan  (p^\diamond_{ki} )\Big)^q ,
    \end{aligned}
\end{equation}
where $\tan(\cdot)$ is the tan function and $q\in[0,1]$ is a regulatory factor. Theoretically, for any input $(I_i,T_i)$ under noise rate $\eta\leq\frac{N-1}{N}$, we can show (see proofs in supplementary material)
\begin{equation}
    C \leq R_{\mathcal{L}_r}(f^*) - R_{\mathcal{L}_r}(f^*_\eta)\leq 0,
     \label{eqbd}
\end{equation}
where $C = 2\eta (A^{(1-q)}_{\min} - A^{(1-q)}_{\max})/(1-\frac{N\eta}{N-1})
\leq 0$.
$C$ increases as $q$ increases and when $q=1$,  $C$ takes the maximum value $0$. $A_{\min}$ and $A_{\max}$ are the maximum and minimum values of $\sum^N_{j=1}\tan(p_{ij})$ under the condition $\sum^N_{j=1}p_{ij}=1$, where $1 < A_{\min} < A_{\max}$, and $0\leq p_{ij} \leq 1$ ($p_{ij} = p^\circ_{ij}\text{ or } p_{ij}^\diamond$). $f^*$ and $f^*_\eta$ are the global minimizers of ${R}_{\mathcal{L}_{r}}(f)$ and ${R}^\eta_{\mathcal{L}_{r}}(f)$, respectively.

\textbf{Analysis:} The larger the $q$ is, $C\rightarrow0$, the tighter the bound of \Cref{eqbd} is. When $q=0$, $\mathcal{L}_r$ is a standard complementary contrastive loss~\cite{hu2023cross}.  
In the extreme case of $q = 1$, ${R}_{\mathcal{L}_{r}}(f^*) = {R}_{\mathcal{L}_{r}}(f^*_\eta)$,
\ie $\mathcal{L}_r$ is noise tolerant since $f^*_\eta$ and $f^*$ are simultaneously the minimizers of ${R}_{\mathcal{L}_{r}}(f)$ (It can also be obtained from \Cref{th2} that $\mathcal{L}_r$ is robust under $q=1$). Put another way, as $q$ approaches 1, the optimum $f^*_\eta$ of the noisy risk will be close to $f^*$ on the clean data more likely, which implies noise tolerance.

Like most robust learning methods~\cite{ma2020normalized,zhang2018generalized}, we should focus more on reliable data, while less on unreliable data. In other words, a smaller value of $q$ is used for more convincing pairs (with larger $\hat{y}_{ii}$), while a larger value of $q$ is for less convincing pairs (with smaller $\hat{y}_{ii}$). Thus, we could empirically utilize the soft corrected label $\hat{y}_{ii}$ to recast $q$ like the soft margin used in NCR~\cite{huang2021learning}, \ie $q = 1-\hat{y}_{ii}$. However, indirect learning will face the underfitting problem, resulting in suboptimal/insufficient performance. To address this issue, we introduce a weighted active learning loss $\mathcal{L}_{d}$ to make the model pay more attention to positive/matched pairs, \ie $ \mathcal{L}_{d}(I_i,T_i,\tilde{y}_{ii}) = - \tilde{y}_{ii} \left( \log p^\circ_{ii} + \log p^\diamond_{ii} \right)$. This positive learning will mine discrimination from direct supervision, which complements the complementary learning loss. By combining active and complementary learning losses, our active complementary loss is defined as:
\begin{equation}
    \mathcal{L}_{acl}(I_i,T_i,\hat{y}_{ii}) =  \mathcal{L}_{d}(I_i,T_i, \hat{y}_{ii}) + \lambda \mathcal{L}_{r}(I_i,T_i, 1-\hat{y}_{ii}),
    \label{eq8}
\end{equation}
where $\lambda$ is a scale factor to prevent $\mathcal{L}_d$ from dominating the cross-modal training and quickly overfitting NC. As shown in \Cref{eq8}, when $(I_i, T_i)$ is a noisy pair and $\hat{y}_{ii}$ ideally approaches 0, the loss emphasizes robust complementary learning, thus mitigating overfitting to NC. Conversely, when $(I_i, T_i)$ is a clean pair and $\hat{y}_{ii}$ ideally approaches 1, the loss focuses on discriminative learning, thereby facilitating the accurate acquisition of visual-semantic associations. However, due to computational resource constraints, we cannot use the entire training set to perform cross-modal learning. Therefore, we relax $N$ to the size $K$ of the mini-batch $\mathbf{x}$ by Monte Carlo sampling. Without loss of generality, the final loss for cross-modal learning is given by:
\begin{equation}
    \mathcal{L}_{acl}(\mathbf{x}) = \frac{1}{K} \sum^K_{i=1} \mathcal{L}_{acl}(I_i,T_i,\hat{y}_{ii}).
    \label{cnlloss}
\end{equation}
\begin{lemma}
     In an instance-level cross-modal matching problem, under uniform NC with noise rate $\eta \leq \frac{N-1}{N}$,  when $q=1$, $\mathcal{L}_r$ is noise tolerant.
     \label{th2}
\end{lemma}
\begin{proof}
    The proofs of \Cref{eqbd} and \Cref{th2} can be found in the supplementary material.
\end{proof}

\subsection{Self-refining Correspondence Correction \label{sec3.4}}
Another key to solving NC is how to obtain accurate correspondence estimations so as to reduce the adverse effects of NC.
To this end, we propose an effective Self-refining Correspondence Correction paradigm (SCC). SCC leverages Momentum Correction (MC) to aggregate historical predictions, providing stable and accurate correspondence estimations while alleviating the over-memorization to NC. To eliminate the error accumulation against NC, we combine multiple independent Self-Refining (SR) in the entire training process. Specifically, the MC for the correspondence of $(I_i,T_i)$ at the $t$-th epoch is defined as follows:
\begin{equation}
    {y}_{ii}^t =  
        \beta {y}^{t-1}_{ii} + (1-\beta) \hat{p}^t(I_i,T_i),
        \label{eq.9}
\end{equation}
where $\beta \in (0,1)$ represents the momentum coefficient, $ \hat{p}^t(I_i,T_i) = (p^\circ_{ii} +  p^\diamond_{ii})/2$ denotes the average matching probability at the $t$-th epoch. Through adequate cross-modal training, our CRCL will obtain a more stable and accurate soft correspondence label by smoothly evolving and expanding the receptive field of correction based on historical predictions with MC. Notably, as training progresses, some pairs would always be incorrectly distinguished as clean or noisy ones, resulting in the error accumulation of the estimated labels (see \Cref{fig1c}). Additionally, even though the updates performed by MC help reduce the negative influence of NC, the initial correspondence label $\left(y^0_{ii}\right)$ still greatly affects the quality of subsequent smooth corrections. In other words, providing more accurate initial correspondences weakens the DNN's memorization against NC, thereby reducing the risk of error accumulation.

\begin{algorithm}[!htb] 
   \caption{The pseudo-code of CRCL}  
   \textbf{Input:} A noisy training dataset $\mathcal{D}_\eta$, image-text matching model $\mathcal{M}(\Theta)$;\\  
   \textbf{Initialize:} $\Theta$;\\
   \For{$e^j$ in $\left[e_1,e_2,\cdots,e_m\right]$}{
        \For{$t$ in $[1,2,\cdots,e^j]$}{
            \For {$\mathbf{x}$ in batches}{
                Obtain the bidirectional matching probabilities of $\mathbf{x}$ with \Cref{eq.3};\\ 
                Update the correspondence labels with \Cref{eq.11};\\
                Obtain the corrected labels with \Cref{eq.12};\\
                Compute the overall loss $\mathcal{L}_{acl}(\mathbf{x})$ with \Cref{cnlloss};\\
                ${\Theta}=\text{Optimizer}\left(\Theta,\mathcal{L}_{acl}(\mathbf{x})\right)$;\\
            }
        }
        Re-initialize  $\Theta$;\\
   }
   \textbf{Output:} The learned parameters $\hat{\Theta}$;  
   \label{alg1}
\end{algorithm}

To achieve accurate initial correspondence estimations, SCC refines updated correspondence labels historically in epochs using MC through multiple concatenated SR pieces during the entire training process. Subsequent SR pieces could gradually aggregate the learned correspondence knowledge from previous pieces, thus improving the quality of estimated correspondences progressively. Furthermore, each SR piece is trained from scratch, which aims to clear accumulated error/noise that has been memorized, thus providing more accurate correspondence predictions for subsequent training. Mathematically, SCC consists of multiple SR pieces ($[e_1,\cdots,e_m]$) based on MC, where each piece undergoes robust learning for $e_{j}$ ($j\in\{1,\cdots,m\}$) epochs. Thus, during the $j$-th SR training, the estimated soft label of the $i$-th pair at $t$-th epoch is reconsidered as follows:
\begin{equation}
y^{(j,t)}_{ii} = \left\{ \begin{array}{ll}
y^{{(j-1,e_{j-1})}}_{ii},&\text{if $t\leq e_f$}, \\ 
\hat{p}^{(j,t-1)}(I_i,T_i), &\text{if $j = 1$ and $t=(e_f+1)$}, \\ 
\beta y^{(j,t-1)}_{ii} + (1-\beta) \hat{p}^{(j,t-1)}(I_i,T_i), & \text{otherwise},
\end{array}\right.
 \label{eq.11}
\end{equation}
where $\hat{p}^{(j,t-1)}(*)$ is the average matching probability of pair $(I_i,T_i)$ at the $(t-1)$-th epoch during the $j$-th training piece, $e_f$ denotes the number of epochs to freeze the correspondence label, preventing insufficient model training in the early stage from affecting the correction quality. In our experiments, we set all initial labels to 1, assuming that all training pairs are matched at the beginning of the first SR piece. In practice, we assign the label of the confident noisy pair as 0 to reduce the risk of producing erroneous supervision information. Therefore, the final corrected correspondence label used for $\mathcal{L}_{acl}$ is defined as:
\begin{equation}
     \hat{y}_{ii} = \left\{ \begin{array}{ll}
        0,&\text{if $ y^{(j,t)}_{ii}< \epsilon$}, \\ 
        y^{(j,t)}_{ii},&\text{otherwise},
\end{array}\right.
     \label{eq.12}
\end{equation}
where $\epsilon=0.1$ is a fixed threshold used to filter the confident noisy pairs in experiments.

\section{Experiments}
In this section, comprehensive experiments are conducted on three widely used benchmarks to demonstrate the robustness and effectiveness of our CRCL under multiple scenarios, including synthetic noise, real-world noise, and well-annotated correspondence.

\subsection{Datasets and Protocols}

\textbf{Datasets:} For an extensive evaluation, we use three benchmark datasets (\ie Flickr30K \cite{young2014image}, MS-COCO \cite{lin2014microsoft} and CC152K \cite{huang2021learning}) in our experiments. More specifically, Flickr30K is a widely-used image-text dataset collected from the Flickr website, which comprises 31,000 images and each one has 5 corresponding textual descriptions. Following \cite{karpathy2015deep}, 30,000 images are employed for training, 1,000 images for validation, and 1,000 images for testing in our experiments. MS-COCO is a large-scale image-text dataset, which has 123,287 images, and 5 captions are given to describe each image. We follow the split of \cite{karpathy2015deep,lee2018stacked} to carry out our experiments, \ie 5000 validation images, 5000 test images, and the rest for training. CC152K is a subset of Conceptual Captions (CC) \cite{sharma2018conceptual} collected in the real world, which is selected by \cite{huang2021learning}. Due to the absence of manual annotation, there are about 3\% $\thicksim$ 20\% incorrect correspondences in CC, \ie real-world noisy correspondences. CC152K contains 150,000 image-text pairs for training, 1,000 pairs for validation, and 1,000 pairs for testing.

\textbf{Evaluation Protocols:} Recall at K (R@K=1, 5, and 10) is used to measure the performance of bidirectional retrievals, which is defined as the proportion of the queries with the correct item in the top K retrieved results. Besides, flowing~\cite{qin2022deep}, we also take the sum of all Recalls to evaluate the overall performance, \ie rSum.

\subsection{Implementation Details}
Our CRCL is a generalized robust framework that could extend existing methods to confront noisy correspondences. To demonstrate the effectiveness and robustness of CRCL, we extend two representative methods, \ie  VSE$\infty$\cite{chen2021learning} and SGRAF (SGR and SAF) \cite{diao2021similarity}, to perform robust image-text matching, respectively. Specifically, the shared hyper-parameters are set as the same as the original works \cite{diao2021similarity,chen2021learning}, \eg the batch size is 128, the word embedding size is 300, and the joint embedding dimensionality is 1,024. More specific hyper-parameters and implementation details are given in our supplementary material due to the space limitation.

\subsection{Comparison with State-of-the-Arts \label{sec4.3}}
In this section, we evaluate our CRCL by comparing it with 7 state-of-the-art methods on three benchmarks, \ie SCAN (ECCV'18) \cite{lee2018stacked}, SGRAF (SGR and SAF, AAAI'21) \cite{diao2021similarity}, VSE$\infty$ (CVPR'21) \cite{chen2021learning}, NCR (NeurIPS'21) \cite{huang2021learning}, DECL (ACM MM'22) \cite{qin2022deep}, BiCro (CVPR'23)~\cite{yang2023bicro} and 
MSCN (CVPR'23)~\cite{han2023noisy}. For a fair comparison, all tested approaches adopt the same visual features (BUTD features)~\cite{lee2018stacked} and textual backbone Bi-GRU~\cite{schuster1997bidirectional}. To comprehensively investigate the robustness of our method, we artificially inject synthetic false correspondence of different ratios by proportionally shuffling the captions on Flickr30K and MS-COCO like \cite{huang2021learning}, \ie 20\%, 40\%, 60\%, and 80\% noise rates. In addition to synthetic noise, we also evaluate the robustness of tested methods against the real-world noisy correspondences on CC152K. Due to the space limitation, we only provide the results on MS-COCO 5K under well-annotated correspondences in~\Cref{tb5k}. For fairness, like~\cite{qin2022deep,yang2023bicro}, the ensemble results of CRCL-SGRAF are reported in the paper. More extensive comparison results are provided in the supplementary material to fully demonstrate the superiority of CRCL.

\begin{table}[h]
\caption{Performance comparison (R@K(\%) and rSum) of image-text matching on Flickr30K and MS-COCO 1K. The highest scores are shown in \textbf{bold}. `*' means robust methods.}
\label{tb1}
\Large
\centering
\setlength{\abovedisplayskip}{5pt}
\resizebox{\textwidth}{!}{
\begin{tabular}{c|l|ccc|ccc|c|ccc|ccc|c}
\toprule
\multicolumn{2}{c|}{}&\multicolumn{7}{c|}{\centering Flickr30K}&
\multicolumn{7}{c}{\centering MS-COCO 1K}\\
\multicolumn{2}{c|}{}&\multicolumn{3}{c|}{\centering Image $\rightarrow$ Text}&
\multicolumn{3}{c|}{\centering Text $\rightarrow$ Image}&&
\multicolumn{3}{c|}{\centering Image $\rightarrow$ Text}&
\multicolumn{3}{c|}{\centering Text $\rightarrow$ Image}&
\\\midrule
Noise                &        Methods      & R@1  & R@5  & R@10 & R@1  & R@5  & R@10 & rSum  & R@1  & R@5  & R@10 & R@1  & R@5  & R@10 & rSum  \\ \midrule
\multirow{8}{*}{20\%} 
& SCAN&56.4&81.7&89.3&34.2&65.1&75.6&402.3&28.9&64.5&79.5&20.6&55.6&73.5&322.6\\
& SAF& 51.8&79.5&88.3&38.1&66.8&76.6&401.1&41.0&78.4&89.4&38.2&74.0&85.5&406.5\\
& SGR& 61.2&84.3&91.5&44.5&72.1&80.2&433.8&49.1&83.8&92.7&42.5&77.7&88.2&434.0\\
& VSE$\infty$& 69.0 & 89.2&  94.8&  48.8 & 76.3 & 83.8&461.9&73.5&93.3& 97.0&  57.4 &86.5 &92.8&500.5 \\
& NCR* &76.7&93.9&96.9&57.5&82.8&89.2&497.0&77.0&95.6&98.1&61.5&89.3&95.1&516.6\\
& DECL* &75.6&93.8&97.4&58.5&82.9&89.4&497.6&77.1&95.9&98.4&61.6&89.1&95.2&517.3\\
& BiCro*&\textbf{78.1} &94.4& {97.5} &{60.4}& {84.4}& {89.9}& {504.7}& {78.8} &{96.1} &98.6 &63.7 &90.3& 95.7& 523.2\\
& MSCN*& {77.4}&{94.9}&{97.6}&59.6&83.2&89.2&501.9&78.1&\textbf{97.2}&\textbf{98.8}&{64.3}&{90.4}&{95.8}&{524.6}\\
& \textbf{CRCL}*& {77.9}&\textbf{95.4}&\textbf{98.3}&\textbf{60.9}&\textbf{84.7}&\textbf{90.6}&\textbf{507.8}&\textbf{79.6}& {96.1}& {98.7}& \textbf{64.7}& \textbf{90.6} &\textbf{95.9} &\textbf{525.6} \\\midrule
\multirow{8}{*}{40\%} 
& SCAN&29.9& 60.5& 72.5& 16.4& 38.5& 48.6& 266.4&30.1&65.2&79.2&18.9&51.1&69.9&314.4\\
& SAF& 34.3& 65.6& 78.4& 30.1& 58.0& 68.5& 334.9&36.0&74.4&87.0&33.7&69.4&82.5&383.0\\
& SGR& 47.2& 76.4& 83.2& 34.5& 60.3& 70.5& 372.1&43.9&78.3&89.3&37.0&72.8&85.1&406.4\\
& VSE$\infty$& 30.2 &58.3 &70.2 &22.3 &49.6 &62.7&293.3&53.3& 84.3 &92.1&31.4& 63.8& 75.0&399.9\\
& NCR* &75.3& 92.1& 95.2& 56.2& 80.6& 87.4& 486.8&76.5&95.0&98.2&60.7&88.5&95.0&513.9\\
& DECL*&72.5 &93.1 &97.0& 55.8 &81.2& 88.1& 487.7&77.1&95.7&\textbf{98.3}&61.5&89.2&95.3&517.1\\
& BiCro*&74.6 &92.7& 96.2& 55.5& 81.1& 87.4 &487.5 &77.0 &\textbf{95.9} &\textbf{98.3}& 61.8 &89.2 &94.9 &517.1\\
& MSCN*& 74.4& 94.4 &96.9 &57.2 &81.7& 87.6& 492.2&74.8&94.9&98.0&60.3&88.5&94.4&510.9\\
&\textbf{CRCL}*&\textbf{77.8}&\textbf{95.2}&\textbf{98.0}&\textbf{60.0}&\textbf{84.0}&\textbf{90.2}&\textbf{505.2}&\textbf{78.2} &{95.7} &\textbf{98.3}& \textbf{63.3} &\textbf{90.3} &\textbf{95.7}&\textbf{521.5}\\\midrule
\multirow{8}{*}{60\%} 
& SCAN&16.9& 39.3 &53.9 &2.8 &7.4 &11.4 &131.7&27.8&59.8&74.8&16.8&47.8&66.4&293.4\\
& SAF& 28.3& 54.5 &67.5 &22.1 &47.3 &59.0 &278.7&28.2&63.9&79.4&31.1&65.6&80.5&348.7\\
& SGR& 28.7& 58.0 &71.0 &23.8 &49.5 &60.7 &291.7&37.6&73.3&86.3&33.8&68.6&81.7&381.3\\
& VSE$\infty$& 18.0& 44.0& 55.7&15.1&38.5&51.8&223.1 &33.4&64.8&79.1 &26.0 &60.1 &76.3&339.7\\
& NCR*& 68.7 &89.9 &95.5 &52.0 &77.6 &84.9 &468.6&72.7&94.0&97.6&57.9&87.0&94.1&503.3\\
& DECL*& 69.4& 89.4 &95.2& 52.6 &78.8 &85.9 &471.3&73.8&94.7&97.7&59.6&87.9&94.5&508.2\\
& BiCro*&67.6 &90.8 &94.4& 51.2 &77.6& 84.7& 466.3 &73.9& 94.4 &97.8 &58.3 &87.2 &93.9&505.5 \\
& MSCN*&70.4&91.0&94.9&53.4&77.8&84.1&471.6&74.4&95.1&97.9&59.2&87.1&92.8&506.5 \\
& \textbf{CRCL}*&\textbf{73.1}& \textbf{93.4}& \textbf{95.8}& \textbf{54.8}& \textbf{81.9}& \textbf{88.3} & \textbf{487.3}& \textbf{76.3}& \textbf{95.1 }&\textbf{97.9}& \textbf{60.8} &\textbf{89.0}& \textbf{95.1} &\textbf{514.2}\\\midrule

\multirow{8}{*}{80\%} 
& SCAN&5.1&18.1&27.3&3.9&13.1&19.1&86.6&22.2&51.9&67.5&13.8&41.1&58.6&255.1\\
& SAF& 12.2&32.8&48.4&11.8&30.5&41.5&177.2&24.2& 57.5 &74.1& 24.7& 57.1 &73.0 &310.6\\
& SGR& 13.7&35.1&47.6&12.1&30.9&41.9&181.3&26.7& 60.7 &75.6& 25.3 &58.2 &72.6 &319.1\\
& VSE$\infty$& 8.1& 23.1& 34.7& 7.4& 22.6& 31.8& 127.7&25.4&55.1&70.6&19.2&50.5&68.0&288.8 \\
& NCR*& 1.4&7.1&11.7&1.5&5.4&9.3&36.4&21.6& 52.6 &67.6 &15.1 &38.1& 49.8& 244.8\\
& DECL*&60.7& 84.6& 91.2& 42.1& 69.6& 78.6&426.8&65.6&91.6&96.6&52.0&83.0&91.3&480.1\\ 
& BiCro*&3.6&13.9&20.5&1.7&7.5&13.0&60.2&40.0&72.6&84.7&22.6&53.0&67.2&340.1 \\
& MSCN*&1.0 &4.4 &9.1& 0.4 &1.4 &2.5& 18.8&66.8& 91.6& 96.2& 52.7& 83.0& 90.9& 481.2 \\
& \textbf{CRCL}*&\textbf{62.3}& \textbf{86.8}& \textbf{92.8} &\textbf{46.0} &\textbf{73.6}&\textbf{ 82.2} &\textbf{443.7}&\textbf{72.7}&\textbf{93.5}&\textbf{97.6}&\textbf{57.5}&\textbf{86.8}&\textbf{93.7}&\textbf{501.8}\\
\bottomrule
\end{tabular}}
\vspace{-0.2cm}
\end{table}


\subsubsection{Results under Synthetic Noisy Correspondences}
For quantitative evaluation under specific noise levels, we conduct all tested methods under four different noise rates (\ie 20\%, 40\%, 60\%, and 80\%) of synthetic noisy correspondences on the Flickr30K and MS-COCO datasets. Quantitative results on Flickr30K and MSCOCO 1K test set are shown in \Cref{tb1}. For MS-COCO, the results are computed by a veraging over 5 folds of 1K test images. From the results, one can see that our CRCL could remarkably outperform the robust baselines (NCR, DECL, BiCro, and MSCN) on most of the metrics, which demonstrates the superior robustness of CRCL against NC. Moreover, our CRCL not only performs well in low noise but also achieves the best performance under high noise, especially 80\% noise, which provides strong evidence for the stability and robustness of our method.

\begin{table}
\Large
\caption{Performance comparison on CC152K and MS-COCO 5K.}
\centering
\setlength{\abovedisplayskip}{5pt}
\resizebox{\textwidth}{!}{
\begin{tabular}{l|ccc|ccc|c|ccc|ccc|c}
\toprule
~&\multicolumn{7}{c|}{\centering CC152K}&
\multicolumn{7}{c}{\centering MS-COCO 5K}\\
~&\multicolumn{3}{c|}{\centering Image $\rightarrow$ Text}&
\multicolumn{3}{c|}{\centering Text $\rightarrow$ Image}&&
\multicolumn{3}{c|}{\centering Image $\rightarrow$ Text}&
\multicolumn{3}{c|}{\centering Text $\rightarrow$ Image}&
\\\midrule
Methods &R@1  &R@5  &R@10  &R@1  &R@5  &R@10&rSum &R@1  &R@5  &R@10 &R@1  &R@5  &R@10 &rSum \\
\midrule
SCAN &30.5  &55.3  &65.3  &26.9  &53.0  &64.7 &295.7&44.7 &75.9  &86.6  &33.3  &63.5  &75.4 &379.4\\
VSE$\infty$&34.0&64.5&\textbf{77.0}&12.9&19.2&21.6&229.2&56.6&83.6&91.4&39.3&69.9&81.1&421.9 \\
SGRAF&32.5&59.5&70.0&32.5&60.7&68.7&323.9&58.8	&84.8	&92.1	&41.6	&70.9	&81.5&429.7 \\
NCR*&{39.5}  &64.5 &73.5  &40.3  &64.6  &73.2    &355.6 &58.2  &84.2  &91.5  &41.7  &{71.0}  &{81.3}  &427.9 \\
{DECL}*   &39.0  &{66.1}  &{75.5}  &{40.7}  &{66.3}  &{76.7} 
&{364.3}&{59.2}  &{84.5}  &91.5  &41.7  &70.6  &81.1 &428.6\\
MSCN*&40.1& 65.7& 76.6& 40.6& 67.4 &76.3& 366.7&-&-&-&-&-&-&-\\
BiCro*& 40.8 &67.2& 76.1 &\textbf{42.1}& 67.6& 76.4& 370.2&59.0&84.4&91.7&42.4&71.2&81.7&430.4\\
\textbf{CRCL}*& \textbf{41.8}&\textbf{67.4} &{76.5}& 41.6& \textbf{68.0}&\textbf{78.4}&\textbf{373.7}& \textbf{61.3}&\textbf{85.8}&\textbf{92.7}&\textbf{43.5}&\textbf{72.6}&\textbf{82.7}&\textbf{438.6}\\
\bottomrule
\end{tabular}}
\label{tb5k}
\vspace{-0.45cm}
\end{table}

\subsubsection{Results under Real Noisy Correspondences}
In addition to synthetic noise,  we carry the comparison experiments on the real-world noisy dataset CC152K. The quantitative results are shown in \Cref{tb5k}. From the results, one can see that our CRCL is superior to all baselines with the best overall performance of 373.7\%, which indicates that the proposed method is robust against real-world noise. Specifically, CRCL outperforms the best baseline BiCro, with absolute performance improvement of 1.0\%, 0.2\%, 0.4\%, 0.4\%, and 2.0\%  across different metrics, except for R@1 in text-to-image retrieval.

\subsubsection{Results under Well-annotated Correspondences}
Besides noisy correspondences, we also evaluate the tested methods trained on the well-aligned MS-COCO dataset for a comprehensive comparison. The results on MS-COCO 5K are shown in \Cref{tb5k}, wherein the results of all baselines are reported by the original papers for a fair comparison, except for the reproduced results of BiCro. From the table, our CRCL remarkably outperforms all baselines in terms of rSum. Specifically, CRCL prevails over the best baselines by 8.2\% on overall performance (\ie rSum) absolutely, which shows that our CRCL is not only suitable for noisy cases but also performs well in well-aligned ones.

\subsection{Comparison to pre-trained model}
In this section, we compare our CRCL to the pre-trained model CLIP~\cite{radford2021learning} to further evaluate its effectiveness in handling NC. CLIP is a well-known large pre-trained model that is trained from scratch on a dataset of 400 million image-text pairs collected from the Internet, which includes a large number of training pairs with real NCs. 
More specifically, following~\cite{huang2021learning}, 
we report the zero-shot results and fine-tuning results under 20\% noise and compare them with that of CRCL-SGRAF in~\Cref{tab_clip}.
From the results, CLIP shows a significant performance drop during fine-tuning under NC. On the contrary, the performance of our CRCL under 20\% noise is even better than the zero-shot result of CLIP (ViT-L/14), which shows the strength and potential of our CRCL in dealing with NC.

\begin{table}[h]
    \centering
    \caption{Comparison with CLIP on MS-COCO 5K.}
    \begin{tabular}{l|l|ccc|ccc|c}\toprule
    &&\multicolumn{3}{c|}{\centering Image $\rightarrow$ Text}&\multicolumn{3}{c|}{\centering Text $\rightarrow$ Image}\\\midrule
         Noise Rate& Methods& R@1& R@5& R@10& R@1& R@5& R@10&rSum  \\\midrule
        \multirow{2}{*}{0\%, Zero-Shot} & CLIP (ViT-L/14)&58.4& 81.5& 88.1& 37.8& 62.4& 72.2&400.4\\
                        & CLIP (ViT-B/32)&50.2& 74.6 &83.6& 30.4& 56.0& 66.8&361.6\\\midrule
        \multirow{3}{*}{20\%, Fine-tune} & CLIP (ViT-L/14)& 36.1&61.3&72.5&22.6&43.2&53.7&289.4\\
        & CLIP (ViT-B/32)&21.4& 49.6& 63.3& 14.8& 37.6& 49.6&236.3\\
                        &\textbf{Our  CRCL}&\textbf{59.3}& \textbf{85.2}& \textbf{91.9}& \textbf{42.9}& \textbf{71.9}& \textbf{82.1}&\textbf{433.3}\\\bottomrule                  
    \end{tabular}
    \label{tab_clip}
    \vspace{-0.3cm}
\end{table}


\subsection{Ablation Study}
In this section, we present ablation studies conducted on Flickr30K with 60\% noise, as shown in \Cref{tab_ab}. From the table, one can find that the full version of our CRCL achieves the best performance, which indicates that each component contributes to our method for performance improvement. By recasting $q$ with the rectified label $\hat{y}_{ii}$ using SCC, $\mathcal{L}_r$ shows higher potential against NC. Moreover, through the combination of active learning and robust complementary learning,  there is further performance improvement, which indicates the effectiveness of our ACL. Note that these results are obtained from a single model for a comprehensive evaluation, and are not ensemble results as shown in~\Cref{tb1,tb5k}.

\begin{table}[h]
\caption{Ablation studies on Flickr30K with 60\% noise.}
\label{tab_ab}
\centering 
\begin{tabular}{lc|ccc|ccc|c}
\toprule
\multicolumn{2}{l|}{Configuration}&\multicolumn{3}{c}{\centering Image $\rightarrow$ Text}&
\multicolumn{3}{c|}{\centering Text $\rightarrow$ Image}\\
loss&SCC&R@1  &R@5  &R@10  &R@1  &R@5  &R@10 &rSum\\\midrule
$\mathcal{L}_{acl}$&\checkmark&\textbf{70.5}&\textbf{91.3}&\textbf{95.6}&\textbf{52.5}&\textbf{79.4}&\textbf{86.8}&\textbf{476.1}\\
$\mathcal{L}_{d}$&\checkmark&69.3&91.1&95.2&51.4&78.2&86.0&471.2\\
$\mathcal{L}_{r}$&\checkmark&67.7& 91.2& 95.0& 52.3& 79.0& 86.4&471.6\\
\multicolumn{2}{l|}{$\mathcal{L}_r(q=1)$}&25.2& 53.9& 65.6& 20.3& 45.0& 57.2&267.2\\
\multicolumn{2}{l|}{$\mathcal{L}_r(q=0)$}&63.5&89.7& 93.8& 48.2& 74.5 &82.0&451.7\\
\multicolumn{2}{l|}{$\mathcal{L}_d$}&65.6 &87.5& 93.2& 47.7& 74.6& 82.7&451.3 \\
\bottomrule
\end{tabular}
\end{table}


\subsection{Visualization Analysis}
To further investigate the generalization and robustness of CRCL, we visually study the retrieval performance of our CRCL-VSE$\infty$, CRCL-SAF, and CRCL-SGR on Flickr30K and MS-COCO with varying noise rates in~\Cref{figv}(a,d). From the figure, one can see that CRCL can enhance the robustness of the original methods and perform well even under high noise. In addition, we provide visualization to intuitively showcase the effectiveness of the proposed SCC. Specifically, we visualize the distributions of both cross-modal similarities and corrected correspondences for noisy training pairs using CRCL-SGR on Flickr30K with 60\% NCs. For comparison, we also visualize the results without SCC, where we replace $\hat{y}_{ii}$ with the prediction $\hat{p}^t(I_i,T_i)$ for the current iteration. The visualizations clearly indicate that SCC prevents the over-accumulation of errors in correspondence correction, thus verifying its effectiveness in mitigating the impact of NC.

\begin{figure} [h]
\centering
\subfloat[][Flickr30K]{
\includegraphics[width=0.32\linewidth]{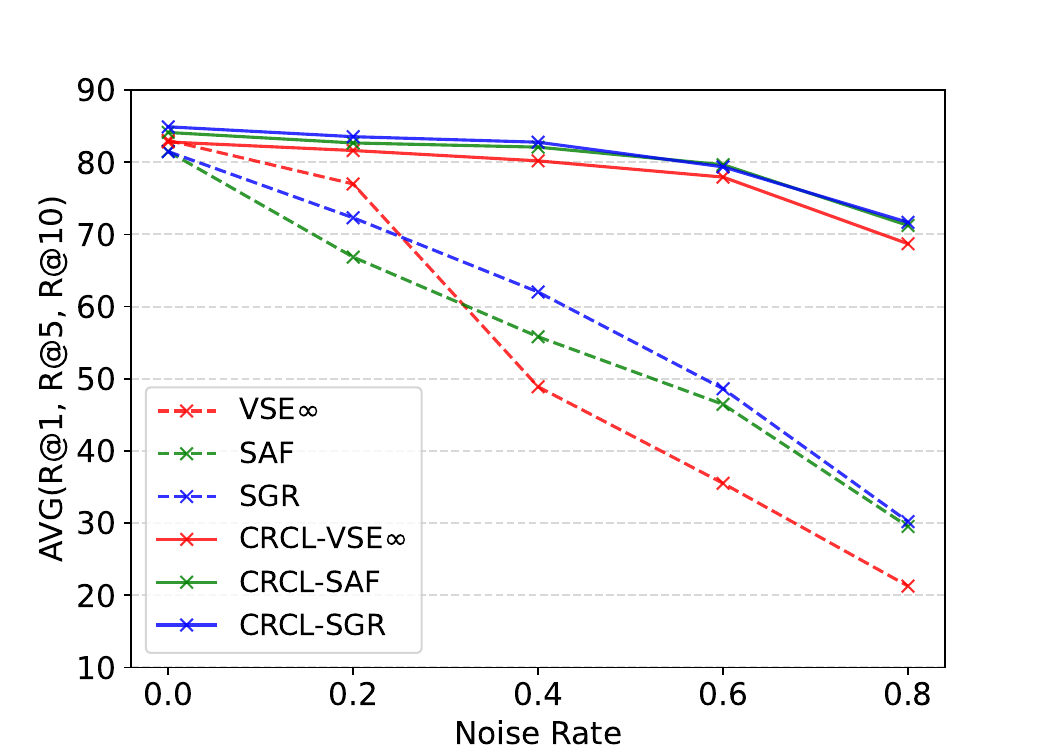}
\label{fig1a}}
\subfloat[][Similarities w/o SCC]{
\includegraphics[width=0.32\linewidth]{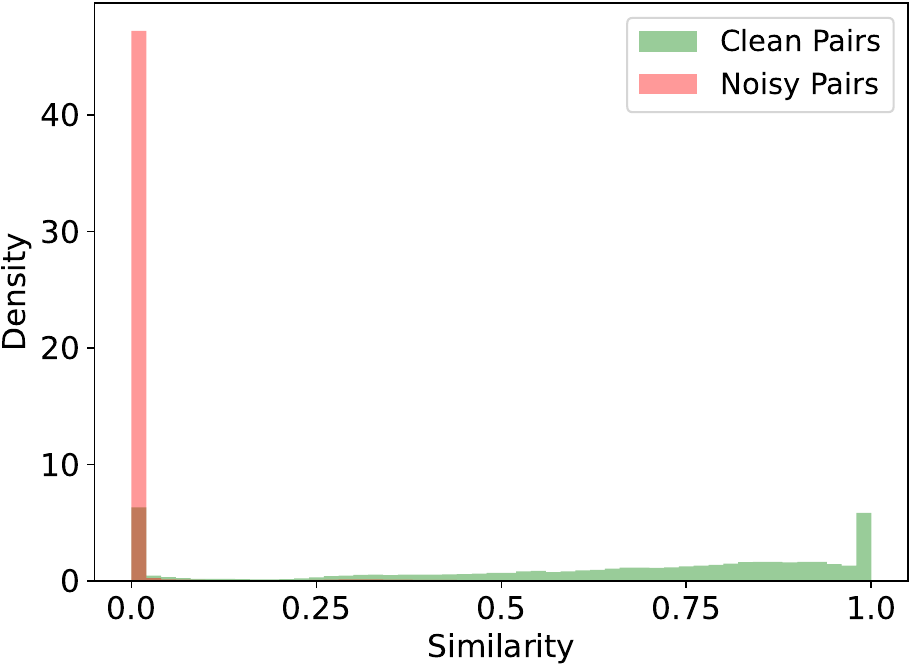}	
\label{fig1b}}
\subfloat[][Correspondences w/o SCC]{
\includegraphics[width=0.32\linewidth]{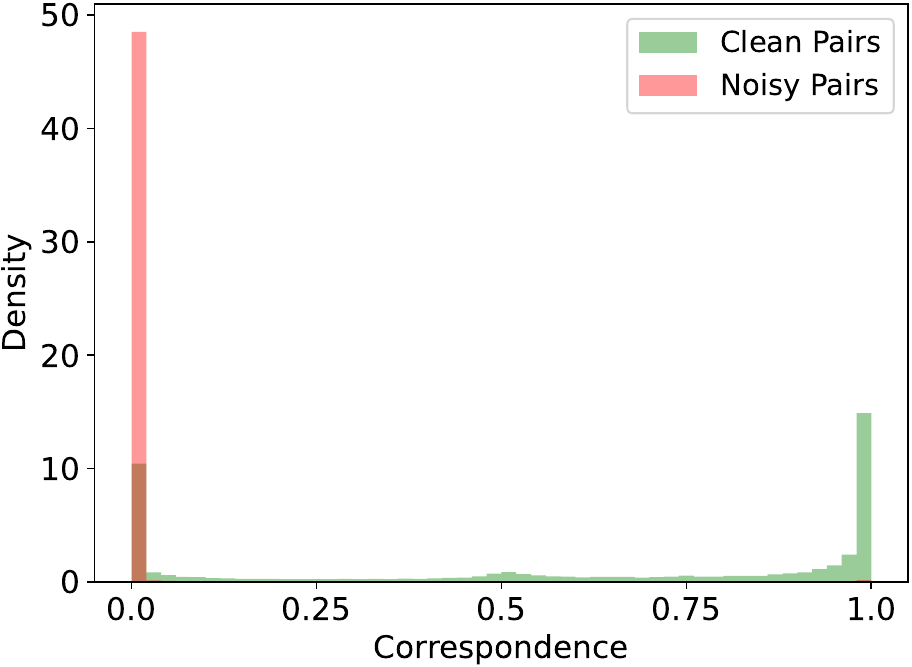}	
\label{fig1c}}
\\
\subfloat[][MS-COCO]{
\includegraphics[width=0.32\linewidth]{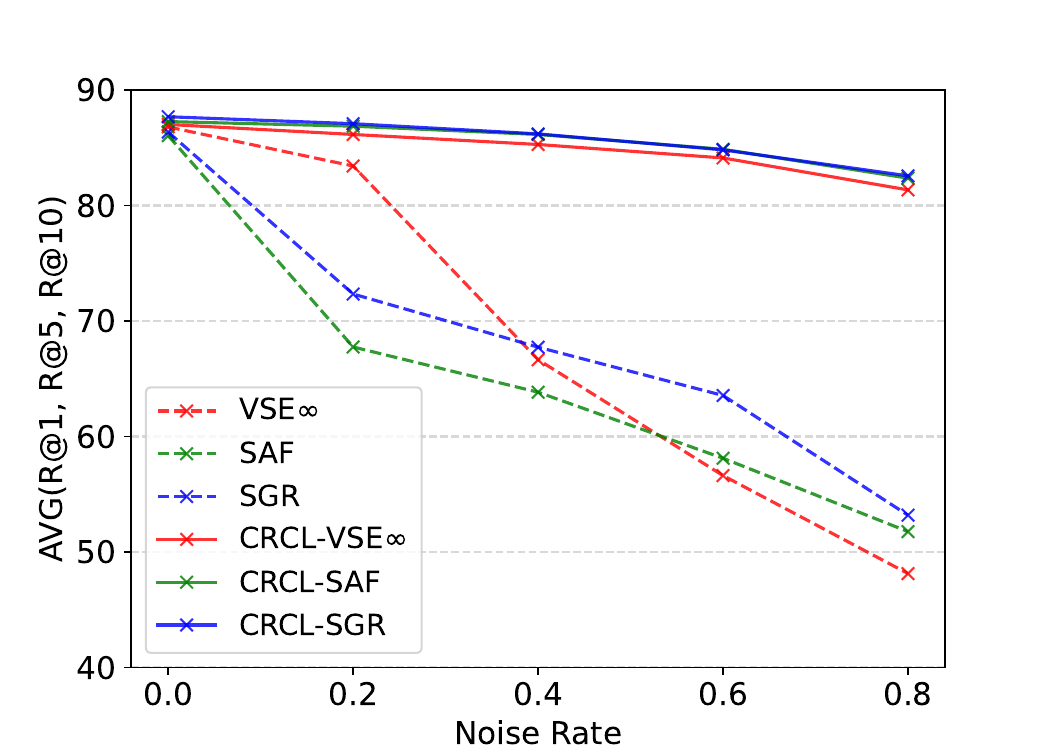}	
\label{fig1d}}
\subfloat[][Similarities with SCC]{
\includegraphics[width=0.32\linewidth]{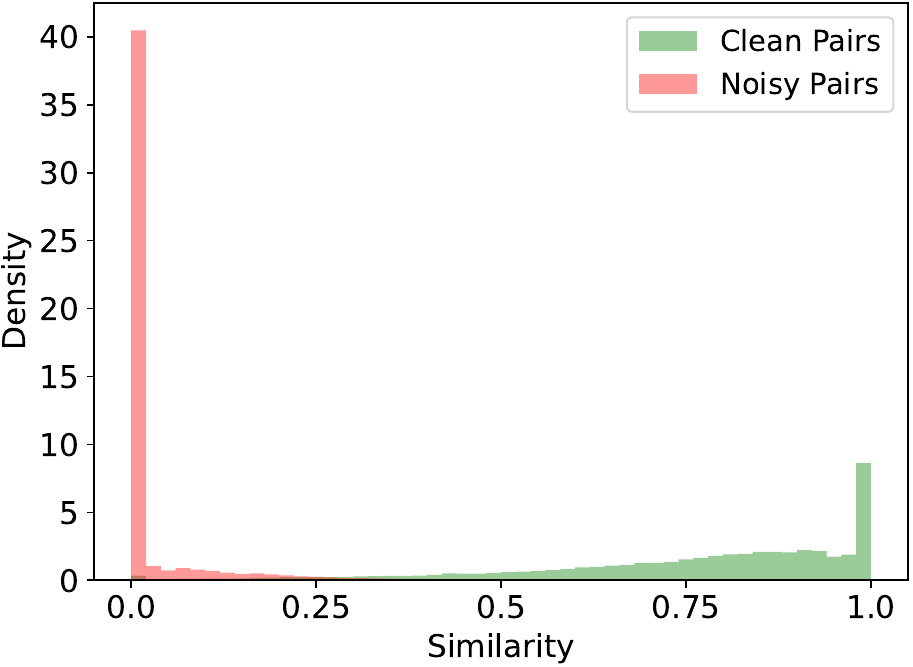}	
\label{fig1e}}
\subfloat[][Correspondences with SCC]{
\includegraphics[width=0.32\linewidth]{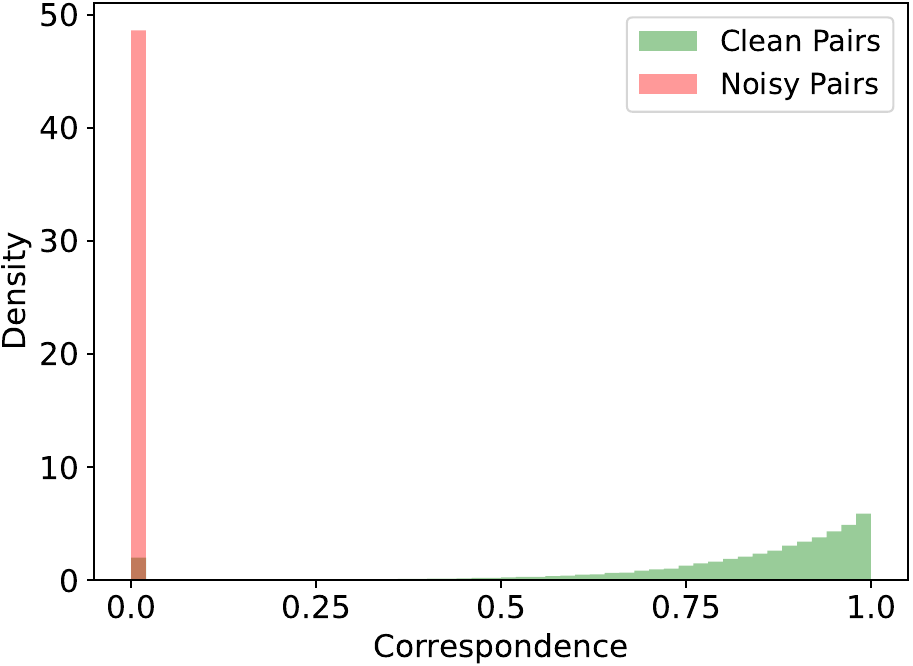}	
\label{fig1f}}
\caption{(a,d) The performance on Flickr30K and MS-COCO with varying noise rates; (b,c/e,f) The similarities and corrected correspondences of training pairs after learning without/with SCC.}
\label{figv}
\vspace{-0.3cm}
\end{figure}



\section{Conclusion}
In this paper, we propose a generalized robust framework CRCL to endow traditional methods with robustness against noisy correspondences. To alleviate the harmful effects brought by NC, we present an Active Complementary Loss (ACL) and a Self-refining Correspondence Correction technique (SCC). These techniques enable effective learning of visual-semantic associations and progressive correction of correspondences based on historical predictions, thus boosting the robustness of image-text matching. Extensive experiments demonstrate that our CRCL achieves state-of-the-art robustness to both synthetic and real-world NCs. Furthermore, ablation studies and visualization analyses further verify the effectiveness and reliability of the proposed components.

\section{Limitations and Broader Impact Statement}
Despite the promising performance of our proposed CRCL, there are some limitations that should be acknowledged. First,  we only study the NC problem between two modalities, \ie image and text. Second,  we did not take into account category-level noise. The proposed CRCL likely impacts various applications that require robust image-text matching, \eg multimedia retrieval, and image annotation. We encourage further study to understand and mitigate the biases and risks potentially brought by image-text matching.

\begin{ack}
    This work was supported by the National Natural Science Foundation of China (U19A2078, U21B2040, 62372315, 62176171, and 62102274), Sichuan Science and Technology Planning Project (2023YFQ0020, 2023YFG0033, 2023ZHCG0016, 2022YFQ0014, 2022YFH0021), Chengdu Science and Technology Project (2023-XT00-00004-GX), the SCU-LuZhou Sciences and Technology Coorperation Program (2023CDLZ-16), and Fundamental Research Funds for the Central Universities under Grant YJ202140.
\end{ack}


\bibliography{ref}{}

\begin{thebibliography}{10}

\bibitem{malinowski2015ask}
Mateusz Malinowski, Marcus Rohrbach, and Mario Fritz.
\newblock Ask your neurons: A neural-based approach to answering questions about images.
\newblock In {\em Proceedings of the IEEE international conference on computer vision}, pages 1--9, 2015.

\bibitem{vinyals2015show}
Oriol Vinyals, Alexander Toshev, Samy Bengio, and Dumitru Erhan.
\newblock Show and tell: A neural image caption generator.
\newblock In {\em Proceedings of the IEEE conference on computer vision and pattern recognition}, pages 3156--3164, 2015.

\bibitem{xu2015show}
Kelvin Xu, Jimmy Ba, Ryan Kiros, Kyunghyun Cho, Aaron Courville, Ruslan Salakhudinov, Rich Zemel, and Yoshua Bengio.
\newblock Show, attend and tell: Neural image caption generation with visual attention.
\newblock In {\em International conference on machine learning}, pages 2048--2057. PMLR, 2015.

\bibitem{diao2021similarity}
Haiwen Diao, Ying Zhang, Lin Ma, and Huchuan Lu.
\newblock Similarity reasoning and filtration for image-text matching.
\newblock In {\em Proceedings of the AAAI Conference on Artificial Intelligence}, volume~35, pages 1218--1226, 2021.

\bibitem{jiang2023cross}
Ding Jiang and Mang Ye.
\newblock Cross-modal implicit relation reasoning and aligning for text-to-image person retrieval.
\newblock {\em arXiv preprint arXiv:2303.12501}, 2023.

\bibitem{9782584}
Peng Hu, Hongyuan Zhu, Jie Lin, Dezhong Peng, Yin-Ping Zhao, and Xi~Peng.
\newblock Unsupervised contrastive cross-modal hashing.
\newblock {\em IEEE Transactions on Pattern Analysis and Machine Intelligence}, pages 1--1, 2022.

\bibitem{faghri2017vse++}
Fartash Faghri, David~J Fleet, Jamie~Ryan Kiros, and Sanja Fidler.
\newblock Vse++: Improving visual-semantic embeddings with hard negatives.
\newblock {\em arXiv preprint arXiv:1707.05612}, 2017.

\bibitem{lee2018stacked}
Kuang-Huei Lee, Xi~Chen, Gang Hua, Houdong Hu, and Xiaodong He.
\newblock Stacked cross attention for image-text matching.
\newblock In {\em Proceedings of the European conference on computer vision (ECCV)}, pages 201--216, 2018.

\bibitem{chen2021learning}
Jiacheng Chen, Hexiang Hu, Hao Wu, Yuning Jiang, and Changhu Wang.
\newblock Learning the best pooling strategy for visual semantic embedding.
\newblock In {\em Proceedings of the IEEE/CVF conference on computer vision and pattern recognition}, pages 15789--15798, 2021.

\bibitem{sharma2018conceptual}
Piyush Sharma, Nan Ding, Sebastian Goodman, and Radu Soricut.
\newblock Conceptual captions: A cleaned, hypernymed, image alt-text dataset for automatic image captioning.
\newblock In {\em Proceedings of the 56th Annual Meeting of the Association for Computational Linguistics (Volume 1: Long Papers)}, pages 2556--2565, 2018.

\bibitem{jia2021scaling}
Chao Jia, Yinfei Yang, Ye~Xia, Yi-Ting Chen, Zarana Parekh, Hieu Pham, Quoc Le, Yun-Hsuan Sung, Zhen Li, and Tom Duerig.
\newblock Scaling up visual and vision-language representation learning with noisy text supervision.
\newblock In {\em International Conference on Machine Learning}, pages 4904--4916. PMLR, 2021.

\bibitem{huang2021learning}
Zhenyu Huang, Guocheng Niu, Xiao Liu, Wenbiao Ding, Xinyan Xiao, Hua Wu, and Xi~Peng.
\newblock Learning with noisy correspondence for cross-modal matching.
\newblock {\em Advances in Neural Information Processing Systems}, 34:29406--29419, 2021.

\bibitem{yang2022robust}
Mouxing Yang, Yunfan Li, Peng Hu, Jinfeng Bai, Jiancheng Lv, and Xi~Peng.
\newblock Robust multi-view clustering with incomplete information.
\newblock {\em IEEE Transactions on Pattern Analysis and Machine Intelligence}, 45(1):1055--1069, 2022.

\bibitem{yang2021partially}
Mouxing Yang, Yunfan Li, Zhenyu Huang, Zitao Liu, Peng Hu, and Xi~Peng.
\newblock Partially view-aligned representation learning with noise-robust contrastive loss.
\newblock In {\em Proceedings of the IEEE/CVF conference on computer vision and pattern recognition}, pages 1134--1143, 2021.

\bibitem{wen2023deep}
Jie Wen, Chengliang Liu, Shijie Deng, Yicheng Liu, Lunke Fei, Ke~Yan, and Yong Xu.
\newblock Deep double incomplete multi-view multi-label learning with incomplete labels and missing views.
\newblock {\em IEEE Transactions on Neural Networks and Learning Systems}, 2023.

\bibitem{yang2020adversarial}
Xu~Yang, Cheng Deng, Kun Wei, Junchi Yan, and Wei Liu.
\newblock Adversarial learning for robust deep clustering.
\newblock {\em Advances in Neural Information Processing Systems}, 33:9098--9108, 2020.

\bibitem{zhang2023robust}
Huaiwen Zhang, Yang Yang, Fan Qi, Shengsheng Qian, and Changsheng Xu.
\newblock Robust video-text retrieval via noisy pair calibration.
\newblock {\em IEEE Transactions on Multimedia}, 2023.

\bibitem{yang2022learning}
Mouxing Yang, Zhenyu Huang, Peng Hu, Taihao Li, Jiancheng Lv, and Xi~Peng.
\newblock Learning with twin noisy labels for visible-infrared person re-identification.
\newblock In {\em Proceedings of the IEEE/CVF Conference on Computer Vision and Pattern Recognition}, pages 14308--14317, 2022.

\bibitem{qin2022deep}
Yang Qin, Dezhong Peng, Xi~Peng, Xu~Wang, and Peng Hu.
\newblock Deep evidential learning with noisy correspondence for cross-modal retrieval.
\newblock In {\em Proceedings of the 30th ACM International Conference on Multimedia}, pages 4948--4956, 2022.

\bibitem{hu2023cross}
Peng Hu, Zhenyu Huang, Dezhong Peng, Xu~Wang, and Xi~Peng.
\newblock Cross-modal retrieval with partially mismatched pairs.
\newblock {\em IEEE Transactions on Pattern Analysis and Machine Intelligence}, pages 1--15, 2023.

\bibitem{li2020dividemix}
Junnan Li, Richard Socher, and Steven~CH Hoi.
\newblock Dividemix: Learning with noisy labels as semi-supervised learning.
\newblock {\em arXiv preprint arXiv:2002.07394}, 2020.

\bibitem{han2018co}
Bo~Han, Quanming Yao, Xingrui Yu, Gang Niu, Miao Xu, Weihua Hu, Ivor Tsang, and Masashi Sugiyama.
\newblock Co-teaching: Robust training of deep neural networks with extremely noisy labels.
\newblock {\em Advances in neural information processing systems}, 31, 2018.

\bibitem{yang2023bicro}
Shuo Yang, Zhaopan Xu, Kai Wang, Yang You, Hongxun Yao, Tongliang Liu, and Min Xu.
\newblock Bicro: Noisy correspondence rectification for multi-modality data via bi-directional cross-modal similarity consistency.
\newblock {\em arXiv preprint arXiv:2303.12419}, 2023.

\bibitem{han2023noisy}
Haochen Han, Kaiyao Miao, Qinghua Zheng, and Minnan Luo.
\newblock Noisy correspondence learning with meta similarity correction.
\newblock {\em arXiv preprint arXiv:2304.06275}, 2023.

\bibitem{arpit2017closer}
Devansh Arpit, Stanis{\l}aw Jastrz{\c{e}}bski, Nicolas Ballas, David Krueger, Emmanuel Bengio, Maxinder~S. Kanwal, Tegan Maharaj, Asja Fischer, Aaron Courville, Yoshua Bengio, and Simon Lacoste-Julien.
\newblock A closer look at memorization in deep networks.
\newblock In Doina Precup and Yee~Whye Teh, editors, {\em Proceedings of the 34th International Conference on Machine Learning}, volume~70 of {\em Proceedings of Machine Learning Research}, pages 233--242. PMLR, 06--11 Aug 2017.

\bibitem{ghosh2017robust}
Aritra Ghosh, Himanshu Kumar, and P~Shanti Sastry.
\newblock Robust loss functions under label noise for deep neural networks.
\newblock In {\em Proceedings of the AAAI conference on artificial intelligence}, volume~31, 2017.

\bibitem{dong2019dual}
Jianfeng Dong, Xirong Li, Chaoxi Xu, Shouling Ji, Yuan He, Gang Yang, and Xun Wang.
\newblock Dual encoding for zero-example video retrieval.
\newblock In {\em Proceedings of the IEEE/CVF conference on computer vision and pattern recognition}, pages 9346--9355, 2019.

\bibitem{manwani2013noise}
Naresh Manwani and PS~Sastry.
\newblock Noise tolerance under risk minimization.
\newblock {\em IEEE transactions on cybernetics}, 43(3):1146--1151, 2013.

\bibitem{wu2018unsupervised}
Zhirong Wu, Yuanjun Xiong, Stella~X Yu, and Dahua Lin.
\newblock Unsupervised feature learning via non-parametric instance discrimination.
\newblock In {\em Proceedings of the IEEE conference on computer vision and pattern recognition}, pages 3733--3742, 2018.

\bibitem{caron2020unsupervised}
Mathilde Caron, Ishan Misra, Julien Mairal, Priya Goyal, Piotr Bojanowski, and Armand Joulin.
\newblock Unsupervised learning of visual features by contrasting cluster assignments.
\newblock {\em Advances in Neural Information Processing Systems}, 33:9912--9924, 2020.

\bibitem{hu2021learning}
Peng Hu, Xi~Peng, Hongyuan Zhu, Liangli Zhen, and Jie Lin.
\newblock Learning cross-modal retrieval with noisy labels.
\newblock In {\em Proceedings of the IEEE/CVF Conference on Computer Vision and Pattern Recognition}, pages 5403--5413, 2021.

\bibitem{ma2020normalized}
Xingjun Ma, Hanxun Huang, Yisen Wang, Simone Romano, Sarah Erfani, and James Bailey.
\newblock Normalized loss functions for deep learning with noisy labels.
\newblock In {\em International conference on machine learning}, pages 6543--6553. PMLR, 2020.

\bibitem{zhang2018generalized}
Zhilu Zhang and Mert Sabuncu.
\newblock Generalized cross entropy loss for training deep neural networks with noisy labels.
\newblock {\em Advances in neural information processing systems}, 31, 2018.

\bibitem{young2014image}
Peter Young, Alice Lai, Micah Hodosh, and Julia Hockenmaier.
\newblock From image descriptions to visual denotations: New similarity metrics for semantic inference over event descriptions.
\newblock {\em Transactions of the Association for Computational Linguistics}, 2:67--78, 2014.

\bibitem{lin2014microsoft}
Tsung-Yi Lin, Michael Maire, Serge Belongie, James Hays, Pietro Perona, Deva Ramanan, Piotr Doll{\'a}r, and C~Lawrence Zitnick.
\newblock Microsoft coco: Common objects in context.
\newblock In {\em European conference on computer vision}, pages 740--755. Springer, 2014.

\bibitem{karpathy2015deep}
Andrej Karpathy and Li~Fei-Fei.
\newblock Deep visual-semantic alignments for generating image descriptions.
\newblock In {\em Proceedings of the IEEE conference on computer vision and pattern recognition}, pages 3128--3137, 2015.

\bibitem{schuster1997bidirectional}
Mike Schuster and Kuldip~K Paliwal.
\newblock Bidirectional recurrent neural networks.
\newblock {\em IEEE transactions on Signal Processing}, 45(11):2673--2681, 1997.

\bibitem{radford2021learning}
Alec Radford, Jong~Wook Kim, Chris Hallacy, Aditya Ramesh, Gabriel Goh, Sandhini Agarwal, Girish Sastry, Amanda Askell, Pamela Mishkin, Jack Clark, et~al.
\newblock Learning transferable visual models from natural language supervision.
\newblock In {\em International conference on machine learning}, pages 8748--8763. PMLR, 2021.

\bibitem{bertsekas2014constrained}
Dimitri~P Bertsekas.
\newblock {\em Constrained optimization and Lagrange multiplier methods}.
\newblock Academic press, 2014.

\bibitem{anderson2018bottom}
Peter Anderson, Xiaodong He, Chris Buehler, Damien Teney, Mark Johnson, Stephen Gould, and Lei Zhang.
\newblock Bottom-up and top-down attention for image captioning and visual question answering.
\newblock In {\em Proceedings of the IEEE conference on computer vision and pattern recognition}, pages 6077--6086, 2018.

\bibitem{kiros2014unifying}
Ryan Kiros, Ruslan Salakhutdinov, and Richard~S Zemel.
\newblock Unifying visual-semantic embeddings with multimodal neural language models.
\newblock {\em arXiv preprint arXiv:1411.2539}, 2014.

\bibitem{li2019visual}
Kunpeng Li, Yulun Zhang, Kai Li, Yuanyuan Li, and Yun Fu.
\newblock Visual semantic reasoning for image-text matching.
\newblock In {\em Proceedings of the IEEE/CVF international conference on computer vision}, pages 4654--4662, 2019.

\bibitem{wang2020consensus}
Haoran Wang, Ying Zhang, Zhong Ji, Yanwei Pang, and Lin Ma.
\newblock Consensus-aware visual-semantic embedding for image-text matching.
\newblock In {\em Computer Vision--ECCV 2020: 16th European Conference, Glasgow, UK, August 23--28, 2020, Proceedings, Part XXIV 16}, pages 18--34. Springer, 2020.

\bibitem{li2022multi}
Zheng Li, Caili Guo, Zerun Feng, Jenq-Neng Hwang, and Xijun Xue.
\newblock Multi-view visual semantic embedding.
\newblock IJCAI, 2022.

\bibitem{wang2019camp}
Zihao Wang, Xihui Liu, Hongsheng Li, Lu~Sheng, Junjie Yan, Xiaogang Wang, and Jing Shao.
\newblock Camp: Cross-modal adaptive message passing for text-image retrieval.
\newblock In {\em Proceedings of the IEEE/CVF international conference on computer vision}, pages 5764--5773, 2019.

\bibitem{chen2020imram}
Hui Chen, Guiguang Ding, Xudong Liu, Zijia Lin, Ji~Liu, and Jungong Han.
\newblock Imram: Iterative matching with recurrent attention memory for cross-modal image-text retrieval.
\newblock In {\em Proceedings of the IEEE/CVF conference on computer vision and pattern recognition}, pages 12655--12663, 2020.

\bibitem{liu2020graph}
Chunxiao Liu, Zhendong Mao, Tianzhu Zhang, Hongtao Xie, Bin Wang, and Yongdong Zhang.
\newblock Graph structured network for image-text matching.
\newblock In {\em Proceedings of the IEEE/CVF conference on computer vision and pattern recognition}, pages 10921--10930, 2020.

\bibitem{cheng2022cross}
Yuhao Cheng, Xiaoguang Zhu, Jiuchao Qian, Fei Wen, and Peilin Liu.
\newblock Cross-modal graph matching network for image-text retrieval.
\newblock {\em ACM Transactions on Multimedia Computing, Communications, and Applications (TOMM)}, 18(4):1--23, 2022.

\bibitem{li2022image}
Kunpeng Li, Yulun Zhang, Kai Li, Yuanyuan Li, and Yun Fu.
\newblock Image-text embedding learning via visual and textual semantic reasoning.
\newblock {\em IEEE Transactions on Pattern Analysis and Machine Intelligence}, 2022.

\bibitem{zhang2022show}
Huatian Zhang, Zhendong Mao, Kun Zhang, and Yongdong Zhang.
\newblock Show your faith: Cross-modal confidence-aware network for image-text matching.
\newblock 2022.

\bibitem{zhang2022negative}
Kun Zhang, Zhendong Mao, Quan Wang, and Yongdong Zhang.
\newblock Negative-aware attention framework for image-text matching.
\newblock In {\em Proceedings of the IEEE/CVF Conference on Computer Vision and Pattern Recognition}, pages 15661--15670, 2022.

\bibitem{chen2023more}
Yuxiao Chen, Jianbo Yuan, Long Zhao, Tianlang Chen, Rui Luo, Larry Davis, and Dimitris~N Metaxas.
\newblock More than just attention: Improving cross-modal attentions with contrastive constraints for image-text matching.
\newblock In {\em Proceedings of the IEEE/CVF Winter Conference on Applications of Computer Vision}, pages 4432--4440, 2023.

\bibitem{liu2022regularizing}
Yang Liu, Hong Liu, Huaqiu Wang, and Mengyuan Liu.
\newblock Regularizing visual semantic embedding with contrastive learning for image-text matching.
\newblock {\em IEEE Signal Processing Letters}, 29:1332--1336, 2022.

\bibitem{huang2022mack}
Yan Huang, Yuming Wang, Yunan Zeng, and Liang Wang.
\newblock Mack: multimodal aligned conceptual knowledge for unpaired image-text matching.
\newblock {\em Advances in Neural Information Processing Systems}, 35:7892--7904, 2022.

\bibitem{goel2022cyclip}
Shashank Goel, Hritik Bansal, Sumit Bhatia, Ryan Rossi, Vishwa Vinay, and Aditya Grover.
\newblock Cyclip: Cyclic contrastive language-image pretraining.
\newblock {\em Advances in Neural Information Processing Systems}, 35:6704--6719, 2022.

\bibitem{pan2023fine}
Zhengxin Pan, Fangyu Wu, and Bailing Zhang.
\newblock Fine-grained image-text matching by cross-modal hard aligning network.
\newblock In {\em Proceedings of the IEEE/CVF Conference on Computer Vision and Pattern Recognition}, pages 19275--19284, 2023.

\bibitem{fu2023learning}
Zheren Fu, Zhendong Mao, Yan Song, and Yongdong Zhang.
\newblock Learning semantic relationship among instances for image-text matching.
\newblock In {\em Proceedings of the IEEE/CVF Conference on Computer Vision and Pattern Recognition}, pages 15159--15168, 2023.

\bibitem{qin2022nim}
Yalan Qin, Chuan Qin, Xinpeng Zhang, Donglian Qi, and Guorui Feng.
\newblock Nim-nets: Noise-aware incomplete multi-view learning networks.
\newblock {\em IEEE Transactions on Image Processing}, 32:175--189, 2022.

\bibitem{Feng_2023_CVPR}
Yanglin Feng, Hongyuan Zhu, Dezhong Peng, Xi~Peng, and Peng Hu.
\newblock Rono: Robust discriminative learning with noisy labels for 2d-3d cross-modal retrieval.
\newblock In {\em Proceedings of the IEEE/CVF Conference on Computer Vision and Pattern Recognition (CVPR)}, pages 11610--11619, June 2023.

\bibitem{lin2023graph}
Yijie Lin, Mouxing Yang, Jun Yu, Peng Hu, Changqing Zhang, and Xi~Peng.
\newblock Graph matching with bi-level noisy correspondence.
\newblock In {\em Proceedings of the IEEE/CVF international conference on computer vision}, 2023.

\bibitem{qin2022maximum}
Yalan Qin, Xinpeng Zhang, Liquan Shen, and Guorui Feng.
\newblock Maximum block energy guided robust subspace clustering.
\newblock {\em IEEE Transactions on Pattern Analysis and Machine Intelligence}, 45(2):2652--2659, 2022.

\bibitem{qin2021semi}
Yalan Qin, Hanzhou Wu, Xinpeng Zhang, and Guorui Feng.
\newblock Semi-supervised structured subspace learning for multi-view clustering.
\newblock {\em IEEE Transactions on Image Processing}, 31:1--14, 2021.

\bibitem{xu2019l_dmi}
Yilun Xu, Peng Cao, Yuqing Kong, and Yizhou Wang.
\newblock L\_dmi: A novel information-theoretic loss function for training deep nets robust to label noise.
\newblock {\em Advances in neural information processing systems}, 32, 2019.

\bibitem{wang2019symmetric}
Yisen Wang, Xingjun Ma, Zaiyi Chen, Yuan Luo, Jinfeng Yi, and James Bailey.
\newblock Symmetric cross entropy for robust learning with noisy labels.
\newblock In {\em Proceedings of the IEEE/CVF International Conference on Computer Vision}, pages 322--330, 2019.

\bibitem{lu2022ensemble}
Yangdi Lu, Yang Bo, and Wenbo He.
\newblock An ensemble model for combating label noise.
\newblock In {\em Proceedings of the Fifteenth ACM International Conference on Web Search and Data Mining}, pages 608--617, 2022.

\bibitem{patrini2017making}
Giorgio Patrini, Alessandro Rozza, Aditya Krishna~Menon, Richard Nock, and Lizhen Qu.
\newblock Making deep neural networks robust to label noise: A loss correction approach.
\newblock In {\em Proceedings of the IEEE conference on computer vision and pattern recognition}, pages 1944--1952, 2017.

\bibitem{tanaka2018joint}
Daiki Tanaka, Daiki Ikami, Toshihiko Yamasaki, and Kiyoharu Aizawa.
\newblock Joint optimization framework for learning with noisy labels.
\newblock In {\em Proceedings of the IEEE conference on computer vision and pattern recognition}, pages 5552--5560, 2018.

\bibitem{lu2022selc}
Yangdi Lu and Wenbo He.
\newblock Selc: Self-ensemble label correction improves learning with noisy labels.
\newblock In Lud~De Raedt, editor, {\em Proceedings of the Thirty-First International Joint Conference on Artificial Intelligence, {IJCAI-22}}, pages 3278--3284. International Joint Conferences on Artificial Intelligence Organization, 7 2022.
\newblock Main Track.

\end{thebibliography}
\bibliographystyle{unsrt}


\section*{Supplementary Material}
\appendix{
In this supplementary material, we provide additional information for CRCL. Specifically, we first give detailed proofs of \Cref{peqbd} and \Cref{pth1} in \Cref{a1}. 
To improve the reproducibility of CRCL, in \Cref{a3}, we provide comprehensive implementation details of our CRCL for different extended baselines (\ie VSE$\infty$\cite{chen2021learning}, SAF\cite{diao2021similarity}, and SGR\cite{diao2021similarity}) on three datasets. In addition, we present richer additional experimental results and analysis in \Cref{a4}, including parameter analysis, progressive analysis, and extra comparison results, to fully verify the effectiveness and superiority of CRCL. Finally, we supplemented related work in \Cref{a2} to further discuss the related research background.


}

\section{Detailed Proofs\label{a1}}
\subsection{Proof for~\Cref{peqbd}}
\begin{equation}
    C \leq R_{\mathcal{L}_r}(f^*) - R_{\mathcal{L}_r}(f^*_\eta)\leq 0,
     \label{peqbd}
\end{equation}
where $C = 2\eta (A^{(1-q)}_{\min} - A^{(1-q)}_{\max})/(1-\frac{N\eta}{N-1}) \leq 0$.
$C$ increases as $q$ increases and when $q=1$,  $C$ takes the maximum value $0$. $A_{\min}$ and $A_{\max}$ are the maximum and minimum values of $\sum^N_{j=1}\tan(p_{ij})$ under the condition $\sum^N_{j=1}p_{ij}=1$, where $1 < A_{\min} < A_{\max}$, and $0\leq p_{ij} \leq 1$ ($p_{ij} = p^\circ_{ij}\text{ or } p_{ij}^\diamond$). $f^*$ and $f^*_\eta$ are the global minimizers of ${R}_{\mathcal{L}_{r}}(f)$ and ${R}^\eta_{\mathcal{L}_{r}}(f)$, respectively. 

\begin{proof}
 Recall that for any $f$, 
$$
\begin{aligned}
    {R}_{\mathcal{L}_{r}}(f) = & {R}_{\mathcal{L}^\diamond_{r}}(f) + {R}_{\mathcal{L}^\circ_{r}}(f)\\
    =&\mathbb{E}_{(I_i,T_\cdot)\sim{\mathcal{D}}} \left[y_{i\cdot}\mathcal{L}^\circ_{r}(I_i,T_\cdot,q) \right] + \mathbb{E}_{(I_\cdot,T_i)\sim{\mathcal{D}}} \left[ y_{\cdot i} \mathcal {L}^\diamond_{r}(T_i,I_\cdot,q) \right].\\
=& \mathbb{E}_{(I_i,T_i)\sim{\mathcal{D}}} \left[\mathcal{L}^\circ_{r}(I_i,T_i,q) \right] + \mathbb{E}_{(I_i,T_i)\sim{\mathcal{D}}} \left[\mathcal {L}^\diamond_{r}(T_i,I_i,q) \right]
\end{aligned}
$$
For uniform noisy correspondence with noise rate $\eta$, we consider the image-to-text direction and have 
$$
\begin{aligned}
    {R}^\eta_{\mathcal{L}^\circ_{r}}(f) =&  \mathbb{E}_{(I_i,T_\cdot)\sim{\mathcal{D}_\eta}}\left[\tilde{y}_{i\cdot}\mathcal{L}^\circ_{r}(I_i,T_\cdot,q) \right]\\
    =&\mathbb{E}_{(I_i,T_\cdot)\sim{\mathcal{D}}}\left[  
    (1-\eta) \mathcal{L}^\circ_{r}(I_i,T_i,q) + \frac{\eta}{N-1} \sum_{j\neq i} \mathcal{L}^\circ_{r}(I_i,T_j,q) \right]\\    =&\mathbb{E}_{(I_i,T_\cdot)\sim{\mathcal{D}}}\left[  
    (1-\eta) \mathcal{L}^\circ_{r}(I_i,T_i,q) + \frac{\eta}{N-1} \left((N-1)\bigtriangleup^{(1-q)} - \mathcal{L}_r^\circ(I_i,T_i,q) \right)   \right]\\   
    =& \mathbb{E}_{(I_i,T_\cdot)\sim{\mathcal{D}}}\left[  
    (1-\frac{N\eta}{N-1}) \mathcal{L}^\circ_{r}(I_i,T_i,q) + \eta \bigtriangleup^{(1-q)} \right],\\
\end{aligned}
$$
where $\bigtriangleup = \sum\limits^N_{j=1}\tan(p_{ij}^\circ)>1$. Since $\bigtriangleup$ has the maximum and minimum values ($A_{\min}$ and $A_{\max}$, we provide a solution in \textbf{Remark}.) under the condition $\sum^N_{j=1}p^\circ_{ij}=1, 0\leq p^\circ_{ij} \leq 1$, for any $p^\circ_{ij}$, we have
$$
   (1-\frac{N\eta}{N-1}){R}_{\mathcal{L}^\circ_{r}}(f) + \eta A_{\min}^{(1-q)} \leq  {R}^\eta_{\mathcal{L}^\circ_{r}}(f) \leq  (1-\frac{N\eta}{N-1}){R}_{\mathcal{L}^\circ_{r}}(f) + \eta A_{\max}^{(1-q)}.
$$

Similarly, the above equation also holds for ${R}_{\mathcal{L}^\diamond_{r}}(f)$ and ${R}^\eta_{\mathcal{L}^\diamond_{r}}(f)$, \ie

$$
   (1-\frac{N\eta}{N-1}){R}_{\mathcal{L}^\diamond_{r}}(f) + \eta A_{\min}^{(1-q)} \leq  {R}^\eta_{\mathcal{L}^\diamond_{r}}(f) \leq  (1-\frac{N\eta}{N-1}){R}_{\mathcal{L}^\diamond_{r}}(f) + \eta A_{\max}^{(1-q)}.
$$

Thus, for ${R}^\eta_{\mathcal{L}_{r}}(f) $ and ${R}_{\mathcal{L}_{r}}(f)$, under $\eta\leq\frac{N-1}{N}$, we have 
$$
   (1-\frac{N\eta}{N-1}){R}_{\mathcal{L}{r}}(f) + 2\eta A_{\min}^{(1-q)} \leq  {R}^\eta_{\mathcal{L}{r}}(f) \leq  (1-\frac{N\eta}{N-1}){R}_{\mathcal{L}{r}}(f) + 2\eta A_{\max}^{(1-q)}.
$$
or equivalently,
$$
    ({R}^\eta_{\mathcal{L}{r}}(f) - 2\eta A_{\max}^{(1-q)}) / (1-\frac{N\eta}{N-1})  \leq  {R}_{\mathcal{L}{r}}(f) \leq  ({R}^\eta_{\mathcal{L}{r}}(f) - 2\eta A_{\min}^{(1-q)}) / (1-\frac{N\eta}{N-1}) .
$$
Thus, for $f^*_\eta$,
\begin{equation}
   R_{\mathcal{L}_r}(f^*) - R_{\mathcal{L}_r}(f^*_\eta)\ge  (R^\eta_{\mathcal{L}_r}(f^*) - R^\eta_{\mathcal{L}_r}(f^*_\eta))/(1-\frac{N\eta}{N-1}) + C \ge  C, 
\label{eq_b2}
\end{equation}
or equivalently,
\begin{equation}
 R^\eta_{\mathcal{L}_r}(f^*) - R^\eta_{\mathcal{L}_r}(f^*_\eta)\leq (1-\frac{N\eta}{N-1}) (R_{\mathcal{L}_r}(f^*) - R_{\mathcal{L}_r}(f^*_\eta)) + C^\prime \leq C^\prime,
\label{eq_b1}
\end{equation}

where $C = 2\eta (A^{(1-q)}_{\min} - A^{(1-q)}_{\max})/(1-\frac{N\eta}{N-1})
\leq 0$, $C^\prime = 2\eta (A^{(1-q)}_{\max} - A^{(1-q)}_{\min}) \ge 0$, $f^*$ is a minimizer of $R_{\mathcal{L}_r}(f)$. Since $f^*_\eta$ and $f^*$ are the minimizers of $R^\eta_{\mathcal{L}_r}(f)$ and $R_{\mathcal{L}_r}(f)$, respectively,
we have $R^\eta_{\mathcal{L}_r}(f^*) - R^\eta_{\mathcal{L}_r}(f^*_\eta)\ge 0$ or 
$R_{\mathcal{L}_r}(f^*) - R_{\mathcal{L}_r}(f^*_\eta)\leq 0$. Besides, it can be seen from~\Cref{fig_c} that $C/C^\prime$ increases/decreases as $q$ increases. In other words,
under $\eta \leq \frac{N-1}{N}$,
the larger $q$ is, the tighter the bound of~\Cref{eq_b1}/\Cref{eq_b2} is. When $q$ is 1, then $R_{\mathcal{L}_r}(f^*) = R_{\mathcal{L}_r}(f^*_\eta)$ or $R^\eta_{\mathcal{L}_r}(f^*) = R^\eta_{\mathcal{L}_r}(f^*_\eta)$.  This completes the proof.

\begin{figure}[h]
  \centering
  \includegraphics[width=0.4\textwidth]{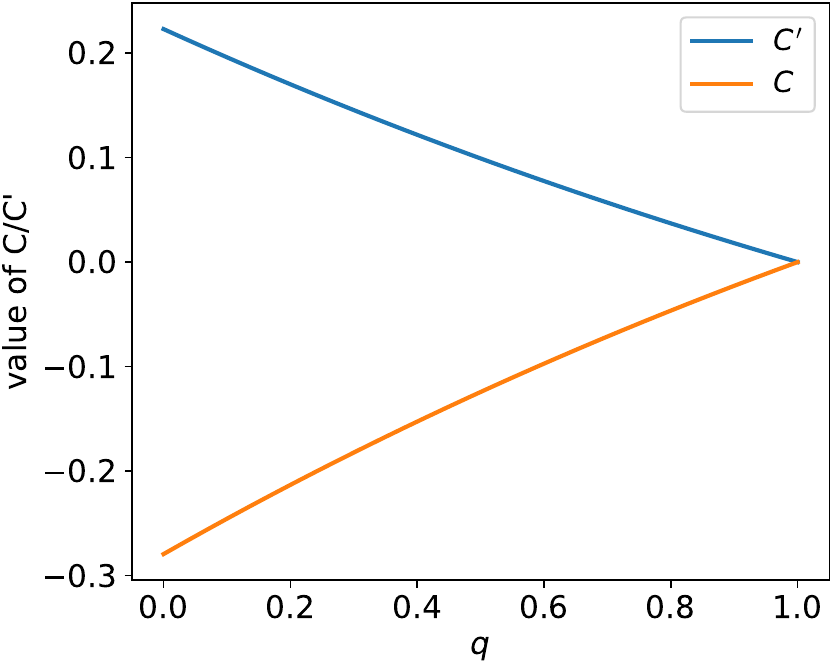}
  \caption{The value of $C/C^\prime$ changes with $q$, wherein $N$ is 100 and $\eta$ is 0.2.}
  \label{fig_c}
\end{figure}



\textbf{Remark}
For the maximum value of $\sum^N_{j=1}\tan(p_{ij})$, For brevity, let $y= \sum^N_{j=1}\tan(p_{ij})=\sum^N_{j=1}\tan(x_j)$ under the condition  $\sum^N_{j=1}x_j=1$ and $0\leq x_j \leq1$. For any $x_i,x_j\in[0,1]$, we have 
$$
\tan(x_i + x_j) = \frac{\tan(x_i)+\tan(x_j)}{1-\tan(x_i)\tan(x_j)}.
$$
Since $\ 0 \leq 1-\tan(x_i)\tan(x_j) \leq 1$, we have
$$
\tan(x_i+x_j) \ge (1-\tan(x_i)\tan(x_j)) \tan(x_i+x_j) = \tan(x_i) + \tan(x_j).
$$
Hence, $\sum^N_{j=1}\tan(x_j)\leq \tan(\sum^N_{j=1}x_j)= \tan1$, \ie $A_{max} = y_{max} = \tan1 \approx 1.5574$.

For the minimum value of $y=\sum^N_{j=1}\tan(x_{j})$, when $x_j\in(0,1)$, $\frac{\partial y}{\partial x_j}=\frac{1}{\cos^2x_j}>0$. Thus, the minimum value does not appear on the hyperplane boundary (0 or 1). We use the Lagrange Multiplier method~\cite{bertsekas2014constrained} to construct the objective as follows
$$
\min y,\ s.t. \sum^N_{j=1} x_j = 1 \Leftrightarrow \min_{x_j,\lambda} (y-\lambda (\sum^N_{j=1}x_j-1) ).
$$
Let $f(x,\lambda)=y-\lambda (\sum^N_{j=1}x_j-1)=\sum^N_{j=1}\tan(x_j)-\lambda (\sum^N_{j=1}x_j-1)$, for $x_j,\lambda$, we have
$$
\left\{\begin{array}{ll}
  \frac{\partial f}{\partial x_j}=\frac{1}{\cos^2 x_j}-\lambda,& j=1,2,\cdots,N,\\
  \frac{\partial f}{\partial \lambda}=\sum^N_{j=1}x_j-1.&
\end{array}\right.
$$
Let $\frac{\partial f}{\partial x_j}=\frac{\partial f}{\partial \lambda}=0$, we have 
$$
\left\{\begin{array}{ll}
  \lambda = \frac{1}{\cos^2x_j}, &j=1,2,\cdots,N,\\
  \sum^N_{j=1}x_j=1.
\end{array}\right.
$$
Thus, when $x_1=x_2=\cdots=x_N=\frac{1}{N}$, $y$ has a minimum value, \ie $A_{min}=y_{min}=N\tan\frac{1}{N}$>1.

\end{proof}

\subsection{Proof for \Cref{pth1}}
\begin{lemma}
     In an instance-level cross-modal matching problem, under uniform NC with noise rate $\eta \leq \frac{N-1}{N}$,  when $q=1$, $\mathcal{L}_r$ is noise tolerant.
     \label{pth1}
\end{lemma}
\begin{proof}  Recall that for any $f$, 
$$
\begin{aligned}
    {R}_{\mathcal{L}_{r}}(f) = & {R}_{\mathcal{L}^\diamond_{r}}(f) + {R}_{\mathcal{L}^\circ_{r}}(f)\\
    =&\mathbb{E}_{(I_i,T_\cdot)\sim{\mathcal{D}}} \left[y_{i\cdot}\mathcal{L}^\circ_{r}(I_i,T_\cdot,q) \right] + \mathbb{E}_{(I_\cdot,T_i)\sim{\mathcal{D}}} \left[ y_{\cdot i} \mathcal {L}^\diamond_{r}(T_i,I_\cdot,q) \right].\\
=& \mathbb{E}_{(I_i,T_i)\sim{\mathcal{D}}} \left[\mathcal{L}^\circ_{r}(I_i,T_i,q) \right] + \mathbb{E}_{(I_i,T_i)\sim{\mathcal{D}}} \left[\mathcal {L}^\diamond_{r}(T_i,I_i,q) \right]
\end{aligned}
$$
Under uniform noisy correspondence with noise rate $\eta$ and $q=1$, for any $f$, ${R}^\eta_{\mathcal{L}^\circ_{r}}(f)$ is written as
\begin{equation}
     \begin{aligned}
      {R}^\eta_{\mathcal{L}^\circ_{r}}(f) =&  \mathbb{E}_{(I_i,T_\cdot)\sim{\mathcal{D}_\eta}}\left[\tilde{y}_{i\cdot}\mathcal{L}^\circ_{r}(I_i,T_\cdot,q=1) \right]\\
    =&\mathbb{E}_{(I_i,T_\cdot)\sim{\mathcal{D}}}[  
    (1-\eta) \mathcal{L}^\circ_{r}(I_i,T_i,q=1) + \frac{\eta}{N-1} \sum_{j\neq i} \mathcal{L}^\circ_{r}(I_i,T_j,q=1) ]\\
        =&\mathbb{E}_{(I_i,T_\cdot)\sim{\mathcal{D}}}\left[  
    (1-\eta) \mathcal{L}^\circ_{r}(I_i,T_i,q=1) + \frac{\eta}{N-1}  \left((N-2)+\frac{\tan{(p^\circ_{ii})}}{\sum^N_{k=1}\tan(p^\circ_{ik})} \right )\right]\\
                =&\mathbb{E}_{(I_i,T_\cdot)\sim{\mathcal{D}}}\left[  
    (1-\eta) \mathcal{L}^\circ_{r}(I_i,T_i,q=1) + \frac{\eta}{N-1}  \left((N-1) -  \mathcal{L}^\circ_{r}(I_i,T_i,q=1)\right )\right]\\
    =& (1-\frac{N\eta}{N-1}) {R}_{\mathcal{L}_r^\circ}(f) + \eta
 \end{aligned}
 \label{eq.9}-
\end{equation}
Note that the equation  between ${R}^\eta_{\mathcal{L}^\diamond_{r}}(f)$ and ${R}_{\mathcal{L}^\diamond_{r}}(f)$  can also be derived similarly as \Cref{eq.9}, \ie $R^\eta_{\mathcal{L}_r^\diamond}=(1-\frac{N\eta}{N-1}) {R}_{\mathcal{L}_r^\diamond}(f) + \eta$. Thus,
$$
{R}^\eta_{\mathcal{L}_{r}}(f) = (1-\frac{N\eta}{N-1}){R}_{\mathcal{L}_r}(f) + 2 \eta
$$
Now, for any $f$, $ {R}^\eta_{\mathcal{L}_{r}}(f^*) -  {R}^\eta_{\mathcal{L}_{r}}(f)=(1 - \frac{N\eta}{N-1})( {R}_{\mathcal{L}_{r}}(f^*) -  {R}_{\mathcal{L}_{r}}(f)) \leq 0$, where  $\eta \leq \frac{N-1}{N}$ and $f^*$ is a globalminimizer of $ {R}_{\mathcal{L}_{r}}(f)$. This proves  $f^*$ is also the global minimizer of ${R}^\eta_{\mathcal{L}_{r}}(f)$.
\end{proof}

\section{Implementation Details\label{a3}}
\subsection{Model Settings}

In this section, we mainly detail the model settings and the implementation of CRCL. To comprehensively verify the effectiveness of our framework, we apply our CRCL to VSE$\infty$~\cite{chen2021learning}, SAF~\cite{diao2021similarity}, and SGR~\cite{diao2021similarity} for further robustness against NC, \ie  CRCL-VSE$\infty$, CRCL-SAF, and  CRCL-SGR. For the VSE model used in our CRCL-VSE$\infty$, we use the same encoder models as VSE$\infty$~\cite{chen2021learning} to project the local region features and word embeddings into the shared common space and then utilize GPO~\cite{chen2021learning} to aggregate local representations into global representations,  wherein the dimensionality of the common space is 1024. For the  CRCL-SAF/SGR, like DECL~\cite{qin2022deep}, we directly perform our CRCL on the similarity output of these models without any changes to their models. In all experiments, we use the same image region features and text backbone for fairness. More specifically, we utilize a Faster R-CNN detection model~\cite{anderson2018bottom} to extract local-level BUTD features of salient regions with top-36 confidence scores for each image, like \cite{lee2018stacked,diao2021similarity}. These features are encoded into a 2,048-dimensional feature vector and then projected into 1,024-dimensional image representations in the common space. For each text, the Bi-GRU language backbone encodes the word tokens into the same dimensional semantic vector space as the image representation.
Following \cite{chen2021learning}, we employ the size augmentation on the training data, which is then fed into the model. For all parameter settings, see \Cref{secps}. The code of our CRCL will be released on GitHub.

\subsection{Parameter Settings \label{secps}}
In this section, we fully provide the parameter settings of our experiments in \Cref{tabp} for easy reproducibility on three benchmark datasets, \ie Flickr30K, MS-COCO, and CC152K. We divide the parameter settings into two groups, the first group includes the parameter settings for the training without synthetic noise (=0\%). The second group consists of the parameter settings for the training under synthetic noise ($>0\%$). Simultaneously, each group details the training parameters of the three extensions of the baselines, \ie CRCL-VSE$\infty$, CRCL-SAF, and CRCL-SGR. Note that the result of CRCL-SGRAF in the paper is the ensemble results of CRCL-SAF and CRCL-SGR. Following \cite{diao2021similarity,qin2022deep,yang2023bicro}, the ensemble strategy is averaging the similarities computed by the two models and then performing image-text matching. 
Next, we will describe these main parameters. $e_f$ represents the number of epochs to freeze the correspondence label, avoiding insufficient model training in the early stage from affecting the correction quality. $e_i$ in $[e_1,\cdots,e_m]$ is the number of training epochs for the $i$-th SR piece. During the last SR piece, CRCL decays the learning rate (lr\_rate) by 0.1 in lr\_update epochs. $\tau$ and $\lambda$ are the temperature parameter and the scale factor in ACL loss, respectively. $\beta$ and $\epsilon$ are the momentum coefficient and the similarity threshold in SCC, respectively. For the parametric analysis of some hyper-parameters, see \Cref{sec.pa} for more details.

\begin{table}[h]
    \caption{The settings of some key parameters for training on three datasets.  }
\centering
 \resizebox{\textwidth}{!}{
    \begin{tabular}{c|c|l|c|c|c|c|c|c|c|c}
    \toprule
    Noise & Datasets & Methods & $e_f$ & $[e_1,\cdots,e_m]$ & lr\_update & lr\_rate & $\beta$ & $\tau$ & $\lambda$&$\epsilon$ \\\midrule
    \multirow{9}{*}{ Synthetic noise $=0\%$} & \multirow{3}{*}{CC152K} & CRCL-VSE$\infty$ & 2 & [7,7,7,32] & 17 & 0.0005 & 0.8 & 0.05&5&0.1 \\
     &  & CRCL-SAF & 2 & [7,7,7,42] & 20 & 0.0005 & 0.8 & 0.05&5&0.1 \\
     &  & CRCL-SGR & 2 & [7,7,7,42] & 20 & 0.0005 & 0.8 & 0.05 &5&0.1\\\cmidrule{2-11}
     & \multirow{3}{*}{{Flickr30K}} & CRCL-{VSE$\infty$} & 2& {[7,7,7,32]} & {15} & {0.0005} & 0.8 & 0.05&5&0.1 \\
     &  & CRCL-SAF & 2 & [7,7,7,32] & 15 & 0.0005 & 0.8 & 0.05 &5&0.1\\
     &  & CRCL-SGR & 2 & [7,7,7,32] & 15 & 0.0005 & 0.8 & 0.05 &5&0.1\\\cmidrule{2-11}
     & \multirow{3}{*}{MS-COCO} & CRCL-VSE$\infty$ & 2 & [4,4,4,22] & 12 & 0.0005 & 0.8 & 0.05 &5&0.1\\
     &  & CRCL-SAF & 2 & [4,4,4,22] & 12 & 0.0005 & 0.8 & 0.05 &5&0.1\\
     &  & CRCL-SGR & 2 & [4,4,4,22] & 12 & 0.0005 & 0.8 & 0.05 &5&0.1\\\midrule
    \multirow{6}{*}{ Synthetic noise $>0\%$} & \multirow{3}{*}{Flickr30K} & CRCL-VSE$\infty$ & 2 & [7,7,7,32] & 15 & 0.0005 & 0.8 & 0.05 &5&0.1\\
     &  & CRCL-SAF & 2 & [7,7,7,32] & 15 & 0.0005 & 0.8 & 0.05 &5&0.1\\
     &  & CRCL-SGR & 2 & [7,7,7,32] & 15 & 0.0005 & 0.8 & 0.05 &5&0.1\\\cmidrule{2-11}
     & \multirow{3}{*}{MS-COCO} & CRCL-VSE$\infty$ &  2 & [4,4,4,22] & 12 & 0.0005 & 0.8 & 0.05&5&0.1 \\
     &  & CRCL-SAF &  2 & [4,4,4,22] & 12 & 0.0005 & 0.8 & 0.05 &5&0.1\\
     &  & CRCL-SGR & 2  & [4,4,4,22] & 12 & 0.0005 & 0.8 & 0.05 &5&0.1\\\bottomrule
    \end{tabular}}
    \label{tabp}
\end{table}

\section{Additional Experiments and Analysis\label{a4}}
\subsection{Parametric Analysis \label{sec.pa}}
The proposed CRCL has three sensitive key hyper-parameters, \ie the temperature parameter $\tau$, the momentum coefficient $\beta$, and the similarity threshold $\epsilon$. Thus, we conduct detailed parameter experiments (shown in \Cref{figv}) on the Flickr30K dataset to evaluate the impact of different hyper-parameter settings and obtain better parameter settings for CRCL. Note that all parametric experiments are performed by CRCL-VSE$\infty$ under 60\% noise. As can be seen from \Cref{fig1a},  too large or too small $\tau$ both cause a performance drop. Thus, in all experiments, we recommend the range of $\tau$  is  0.03 $\sim$ 0.07. From \Cref{fig1b}, when the value of $\beta$ is set to two extreme values, \ie 0 and 1, the performance drops remarkably. Moreover, in the range of $(0,1)$, as $\beta$ increases, the performance gradually improves. We think that with the increase of $\beta$, each correction performed by MC will retain more historical information to reduce perturbation. Thus providing more stable corrected correspondences for training. In all our experiments, $\beta$ is 0.8. From \Cref{fig1c}, we can see that proper filtering is beneficial for mitigating NC. We think this filtering strategy can prevent the active loss from exploiting these confident noisy pairs to produce more misleading gradients. Thus, we set $\epsilon$ as 0.1 in all our experiments.

\begin{figure} [h]

\centering
\subfloat[][The temperature parameter $\tau$]{
\includegraphics[width=0.32\linewidth]{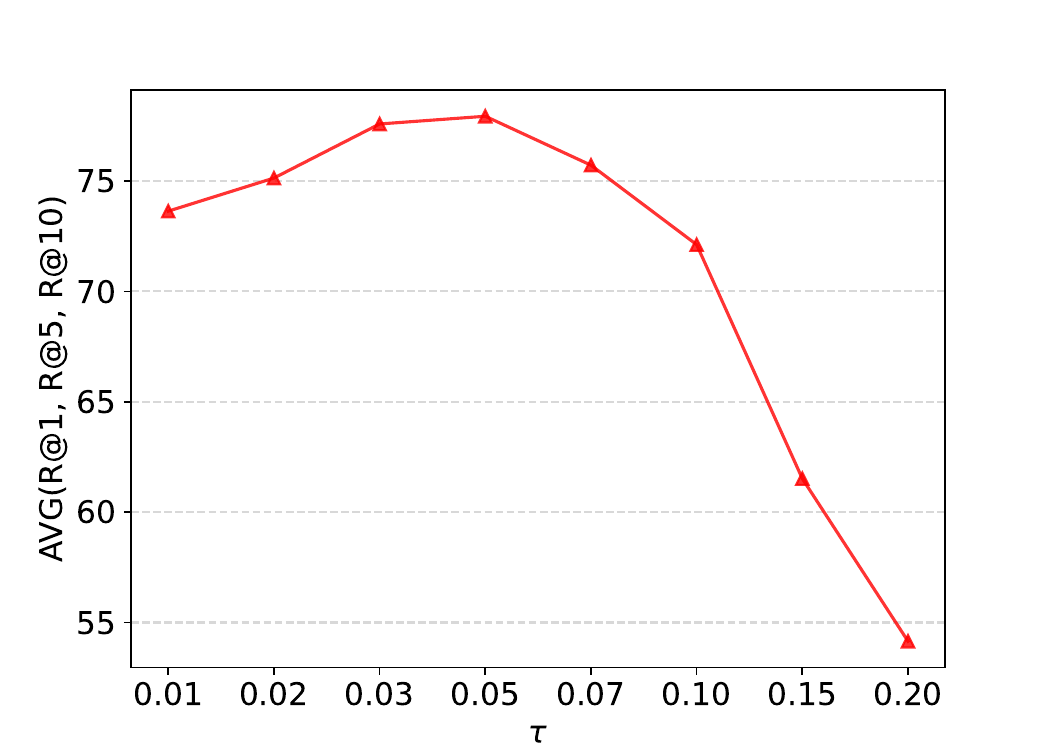}	
\label{fig1a}}
\subfloat[][The momentum coefficient $\beta$]{
\includegraphics[width=0.32\linewidth]{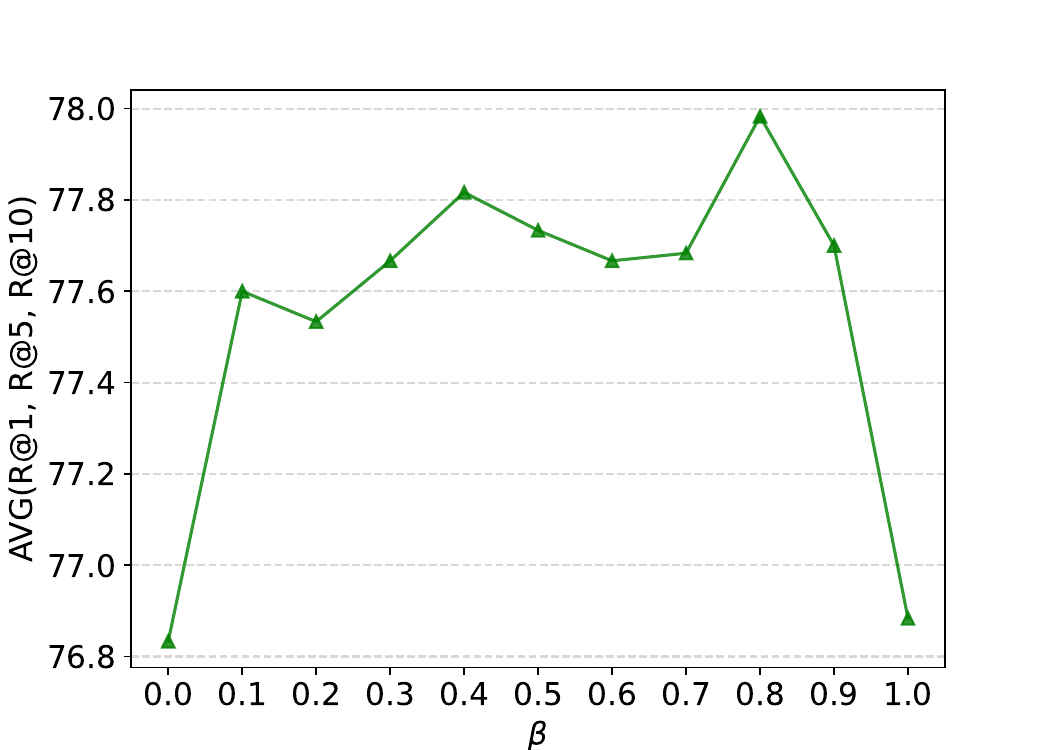}	
\label{fig1b}}
\subfloat[][The similarity threshold $\epsilon$]{
\includegraphics[width=0.32\linewidth]{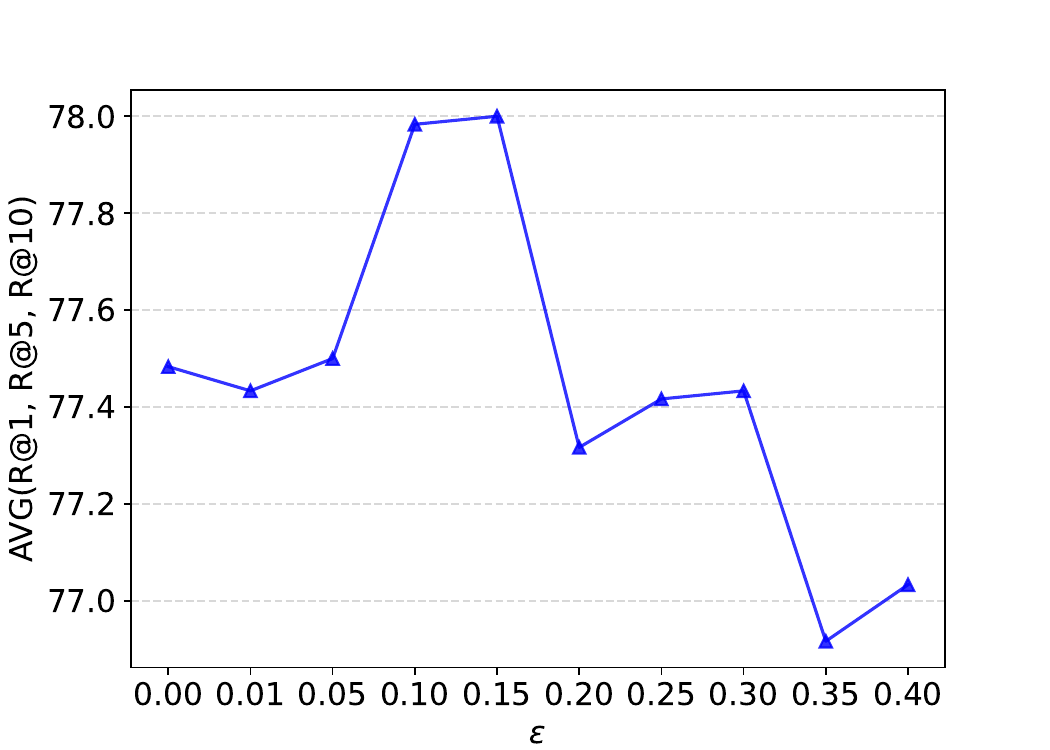}	
\label{fig1c}}
\caption{Parametric analysis on Flickr30K with 60\% noise.}
\label{figv}
\end{figure}

\subsection{Progressive Analysis\label{sec.pa}}
To comprehensively investigate the effectiveness of our CRCL, we carry out some progressive processes to further analyze the advantages of CRCL. Specifically, we recorded the performance of VSE$\infty$ with different loss functions, including CRCL-VSE$\infty$, $\mathcal{L}_d$, $\mathcal{L}_r(q=1)$, $\mathcal{L}_r(q=0)$, Complementary Contrastive Loss (CCL)~\cite{hu2023cross}, the hinge-based Triplet Ranking loss (TR)~\cite{kiros2014unifying}, the Triplet Tanking loss with Hard Negatives (TR-HN)~\cite{faghri2017vse++}, on Flickr30K under 80\% noise. We visualize the performance of bidirectional retrieval in~\Cref{fig2}. From the results, although $\mathcal{L}_r(q=1)$ is noise-tolerant, which is consistent with the theoretical analysis (\cref{pth1}), there would be some underfitting. Our CRCL with ACL loss fully explores the advantages of $\mathcal{L}_r$ and $\mathcal{L}_q$, showing remarkable robustness.

\begin{figure} [h]
\centering
\subfloat[][Image-to-Text]{
\includegraphics[width=0.48\linewidth]{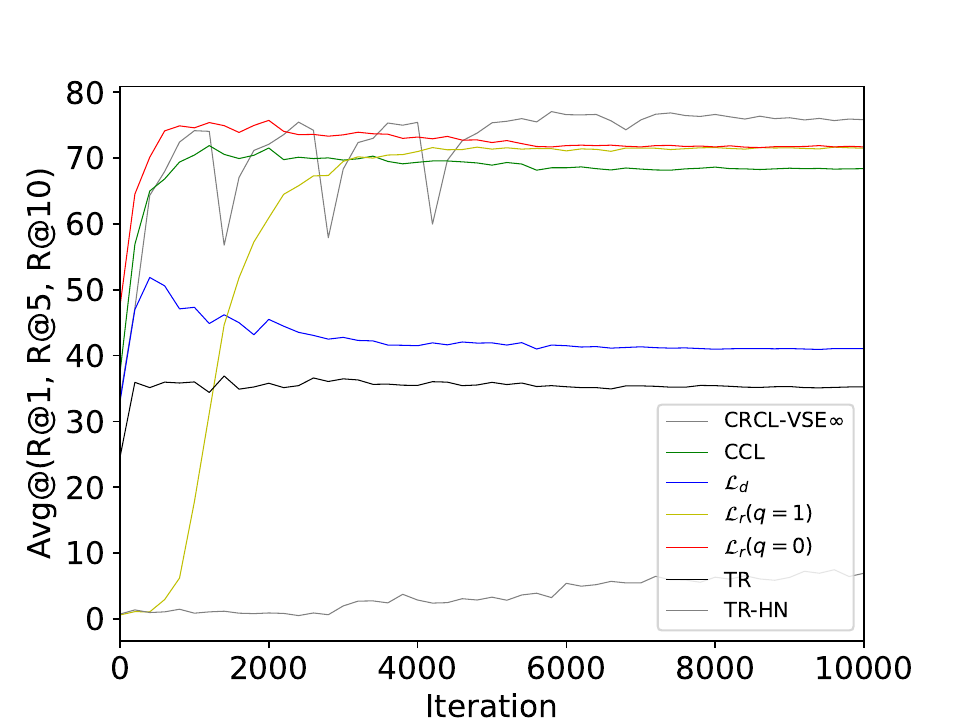}	
\label{fig2a}}
\subfloat[][Text-to-Image]{
\includegraphics[width=0.48\linewidth]{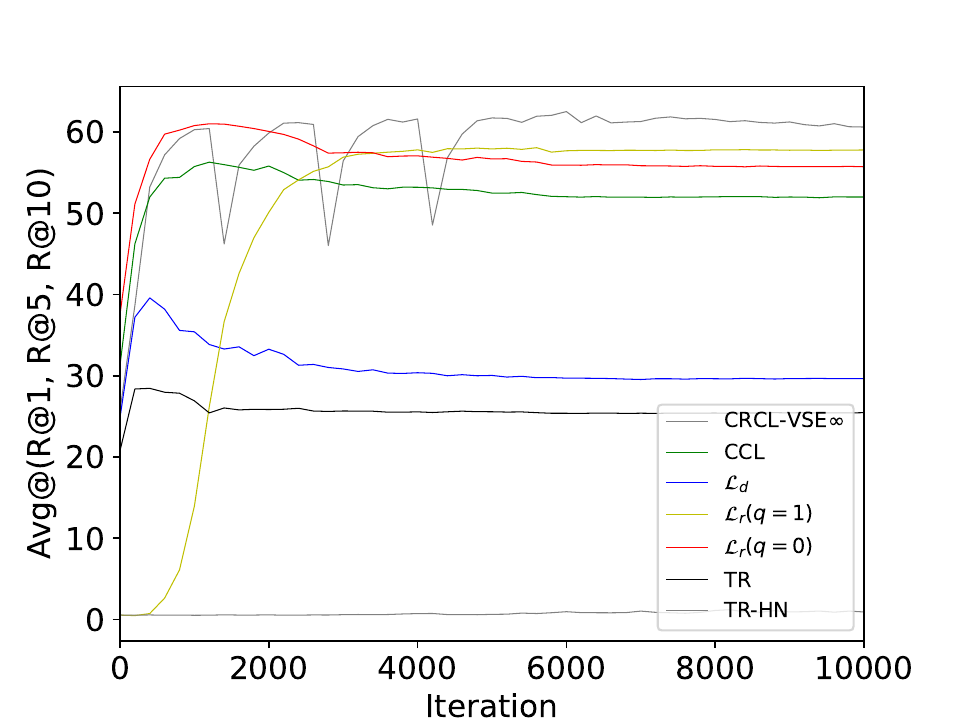}	
\label{fig2b}}
\caption{The performance of VSE$\infty$ with different loss functions.}
\label{fig2}
\end{figure}

\begin{table}[h]
\caption{Performance comparison (R@K(\%) and rSum) of image-text retrieval on Flickr30K and MS-COCO 1K. The highest scores are shown in \textbf{bold}.}
\label{tb1}
\Large
\centering
\setlength{\abovedisplayskip}{5pt}
\resizebox{\textwidth}{!}{
\begin{tabular}{c|l|ccc|ccc|c|ccc|ccc|c}
\toprule
\multicolumn{2}{c|}{}&\multicolumn{7}{c|}{\centering Flickr30K}&
\multicolumn{7}{c}{\centering MS-COCO 1K}\\
\multicolumn{2}{c|}{}&\multicolumn{3}{c|}{\centering Image $\rightarrow$ Text}&
\multicolumn{3}{c|}{\centering Text $\rightarrow$ Image}&&
\multicolumn{3}{c|}{\centering Image $\rightarrow$ Text}&
\multicolumn{3}{c|}{\centering Text $\rightarrow$ Image}&
\\\midrule
Noise                &        Methods      & R@1  & R@5  & R@10 & R@1  & R@5  & R@10 & rSum  & R@1  & R@5  & R@10 & R@1  & R@5  & R@10 & rSum  \\ \midrule
\multirow{10}{*}{20\%} 
& SAF& 51.8&79.5&88.3&38.1&66.8&76.6&401.1&41.0&78.4&89.4&38.2&74.0&85.5&406.5\\
& SGR& 61.2&84.3&91.5&44.5&72.1&80.2&433.8&49.1&83.8&92.7&42.5&77.7&88.2&434.0\\
& VSE$\infty$& 69.0 & 89.2&  94.8&  48.8 & 76.3 & 83.8&461.9&73.5&93.3& 97.0&  57.4 &86.5 &92.8&500.5 \\
&DECL-SAF&73.1&93.0&96.2&57.0&82.0&88.4&489.7&77.2&\textbf{95.9}&98.4&61.6&89.0&95.3&517.4\\
&DECL-SGR&75.4&93.2&96.2&56.8&81.7&88.4&491.7&76.9&95.3&98.2&61.3&89.0&95.1&515.8\\
&BiCro-SAF&\textbf{77.0}&93.3&97.5&57.2&82.3&89.1&496.4&74.5&95.0&98.2&60.7&89.0&95.0 &512.4\\
&BiCro-SGR&76.5&93.1&97.4&58.1&82.4&88.5&496.0&75.7&95.1&98.1&60.5&88.6&94.7&512.7\\
&\textbf{CRCL-VSE$\infty$}& 74.8& 92.8&96.5&55.1&81.8&88.7& 489.7&76.2&95.5&\textbf{98.6}&61.3&89.7&95.6 &516.9 \\
&\textbf{CRCL-SAF}&74.7&93.7& 97.7&57.9& 82.8& 89.2&496.0&78.5 &95.7 &98.5 &63.1 &89.9 &95.5&521.2 \\
&\textbf{CRCL-SGR}&75.8& \textbf{94.6}& \textbf{97.6}& \textbf{59.1}& \textbf{84.0}& \textbf{90.1}&\textbf{501.2}&\textbf{78.9} &{95.7} &98.3 &\textbf{63.6} &\textbf{90.3}& \textbf{95.7}&\textbf{522.5}\\\midrule
\multirow{10}{*}{40\%} 
& SAF& 34.3& 65.6& 78.4& 30.1& 58.0& 68.5& 334.9&36.0&74.4&87.0&33.7&69.4&82.5&383.0\\
& SGR& 47.2& 76.4& 83.2& 34.5& 60.3& 70.5& 372.1&43.9&78.3&89.3&37.0&72.8&85.1&406.4\\
& VSE$\infty$& 30.2 &58.3 &70.2 &22.3 &49.6 &62.7&293.3&53.3& 84.3 &92.1 &31.4 &63.8& 75.0 &399.9\\
&DECL-SAF&72.2&91.4&95.6&54.0&79.4&86.4&479.0&75.8& 95.0 &98.1 &60.3& 88.7 &94.9&512.8\\
&DECL-SGR&72.4&92.2&96.5&54.5&80.1&87.1&482.8&75.9&95.3&98.2&60.2&88.3&94.8&512.7\\
&BiCro-SAF&72.5&91.7&95.3&53.6&79.0&86.4&478.5&75.2&95.0&97.9&59.4&87.9&94.3&509.7\\
&BiCro-SGR&72.8&91.5&94.6&54.7&79.0&86.3&478.9&74.6&94.8&97.7&59.4&87.5&94.0&508.0\\
&\textbf{CRCL-VSE$\infty$}&71.2&92.6&96.3&53.2&80.4&87.4&481.1&74.4&95.1&\textbf{98.4}&59.5&89.1&95.2&511.7 \\
&\textbf{CRCL-SAF}&74.2& 93.8& 97.1& 57.0& 81.8& 88.6&492.5&76.4&\textbf{95.7}&98.1&\textbf{62.1}&89.3&95.3&516.9 \\
&\textbf{CRCL-SGR}&\textbf{75.5}& \textbf{94.0}& \textbf{97.8}& \textbf{57.5}& \textbf{82.6}& \textbf{89.2}&\textbf{496.6}&\textbf{76.8}&95.3&98.2&61.9&\textbf{89.6}&\textbf{95.4}&\textbf{517.2}\\
\midrule
\multirow{10}{*}{60\%} 
& SAF& 28.3& 54.5 &67.5 &22.1 &47.3 &59.0 &278.7&28.2&63.9&79.4&31.1&65.6&80.5&348.7\\
& SGR& 28.7& 58.0 &71.0 &23.8 &49.5 &60.7 &291.7&37.6&73.3&86.3&33.8&68.6&81.7&381.3\\
& VSE$\infty$& 18.0& 44.0& 55.7&15.1&38.5&51.8&223.1 &33.4&64.8&79.1 &26.0 &60.1 &76.3&339.7\\
&DECL-SAF&66.4& 88.1&93.6&49.8&76.1& 84.4&458.4&71.1&93.6&97.3&57.9&86.8&93.8&500.5 \\
&DECL-SGR&68.5&89.9& 94.8&50.3&76.7&84.1&464.3&73.2&94.4&97.9&58.2&86.8&93.9&504.4\\
&BiCro-SAF&67.1&88.3&93.8&48.8&75.2&83.8&457.0&72.5& 94.3& 97.9 &57.7& 86.9 &93.8 &503.1\\
&BiCro-SGR&68.5&89.1&93.1&48.2&74.7&82.7&456.3&73.4& 94.0 &97.5& 58.0& 86.8& 93.6& 503.3\\
&\textbf{CRCL-VSE$\infty$}&68.3& 89.8&\textbf{95.9}& 50.5& 77.8& 85.3&467.6&72.6&94.1&\textbf{98.0}&57.8&87.7&94.5&504.7  \\
&\textbf{CRCL-SAF}&70.1& 90.8& 95.7& \textbf{53.0}& 79.4& \textbf{86.9}&475.9&74.6 &94.5 &97.6 &\textbf{59.5} &\textbf{88.3} &\textbf{94.7}&\textbf{509.2}\\
&\textbf{CRCL-SGR}&\textbf{70.5}&\textbf{ 91.3}& 95.6& 52.5& \textbf{79.4}& 86.8&\textbf{476.1}&\textbf{74.6} &\textbf{94.6}& 97.9 &59.2 &88.0& 94.6&508.9\\\midrule
\multirow{10}{*}{80\%} 
& SAF& 12.2&32.8&48.4&11.8&30.5&41.5&177.2&24.2& 57.5 &74.1& 24.7& 57.1 &73.0 &310.6\\
& SGR& 13.7&35.1&47.6&12.1&30.9&41.9&181.3&26.7& 60.7 &75.6& 25.3 &58.2 &72.6 &319.1\\
& VSE$\infty$& 8.1& 23.1& 34.7& 7.4& 22.6& 31.8& 127.7&25.4&55.1&70.6&19.2&50.5&68.0&288.8 \\
&DECL-SAF&56.3& 82.1& 89.3& 38.7& 64.7& 73.8&404.9&65.9&92.0 &96.6& 52.9& 83.6& 91.7&482.7\\
&DECL-SGR&55.1& 79.8& 87.2& 37.4& 63.4& 72.9&395.8&65.6&91.6 &96.6& 52.0& 83.0& 91.3&480.1\\
&BiCro-SAF&2.4&9.1&15.8&2.4&8.3&13.7&51.7&39.6&72.6&84.7&22.4&52.8&67.1&368.9 \\
&BiCro-SGR&1.7&8.7&13.7&1.3&5.1&8.9&39.4&31.4&62.0&75.2&30.0&60.7&73.2&332.5\\
&\textbf{CRCL-VSE$\infty$}&55.3&82.1&89.1&39.7&68.2&77.8&412.2&67.9&92.8&\textbf{97.1}&53.1&84.7&92.5&488.1\\
&\textbf{CRCL-SAF}&58.4& 83.9& 90.5& \textbf{44.1}& 70.7& 79.8&427.4&\textbf{70.9}&92.8&\textbf{97.1}&55.2&85.3&92.9&494.2  \\
&\textbf{CRCL-SGR}&\textbf{59.2}&\textbf{ 85.1}& \textbf{91.1}& 43.6& \textbf{70.9}& \textbf{80.1}&\textbf{430.0}&70.7&\textbf{92.9}&\textbf{97.1}&\textbf{56.0}&\textbf{85.6}&\textbf{93.1}&\textbf{495.4} \\
\bottomrule
\end{tabular}}
\end{table}

\begin{table}[h]
\caption{Performance comparison (R@K(\%) and rSum) of image-text retrieval on Flickr30K and MS-COCO 1K. The highest scores are shown in \textbf{bold}. * means global-level method.}
\label{tb2}
\Large
\centering
\setlength{\abovedisplayskip}{5pt}
\resizebox{\textwidth}{!}{
\begin{tabular}{l|ccc|ccc|c|ccc|ccc|c}
\toprule
&\multicolumn{7}{c|}{\centering Flickr30K}&
\multicolumn{7}{c}{\centering MS-COCO 1K}\\
&\multicolumn{3}{c|}{\centering Image $\rightarrow$ Text}&
\multicolumn{3}{c|}{\centering Text $\rightarrow$ Image}&&
\multicolumn{3}{c|}{\centering Image $\rightarrow$ Text}&
\multicolumn{3}{c|}{\centering Text $\rightarrow$ Image}&
\\\midrule
Methods      & R@1  & R@5  & R@10 & R@1  & R@5 &  R@10 & rSum  & R@1  & R@5  & R@10 & R@1  & R@5  & R@10 & rSum  \\ \midrule
 VSRN*&71.3 & 90.6 & 96.0 & 54.7 & 81.8 & 88.2 & 482.6& 76.2& 94.8 &98.2& 62.8 &89.7 &95.1& 516.8\\
 CVSE*&70.5& 88.0 &92.7& 54.7 &82.2& 88.6& 476.7&69.2& 93.3 &97.5& 55.7 &86.9& 93.8& 496.4\\
 VSE$\infty$*&76.5 & 94.2 & 97.7 & 56.4 & 83.4 & 89.9 & 498.1& 78.5& 96.0& 98.7&  61.7&  90.3&  95.6 & 520.8\\
 MV-VSE*&{79.0}&{94.9}&{97.7}&{59.1}&{84.6}&{90.6}&{505.9}&{78.7} &{95.7} &{98.7} &62.7& 90.4& {95.7}& {521.9}\\
 SCAN&67.4 & 90.3 & 95.8 & 48.6 & 77.7 & 85.2 & 465.0& 72.7&  94.8&  98.4&  58.8 & 88.4 &94.8&  507.9\\
 CAMP&68.1 & 89.7 & 95.2 & 51.5 & 77.1 & 85.3 & 466.9& 72.3 & 94.8 & 98.3 & 58.5&  87.9&  95.0 & 506.8\\
 IMRAM& 74.1 & 93.0 & 96.6 & 53.9 & 79.4 & 87.2 & 484.2& 76.7&  95.6&  98.5 & 61.7&  89.1&  95.0 & 516.6\\
 GSMN&76.4 & 94.3 & 97.3 & 57.4 & 82.3 & 89.0 & 496.7& 78.4&  96.4&  98.6 & 63.3&  90.1&  95.7 & 522.5\\
 SGRAF&77.8 & 94.1 & 97.4 & 58.5 & 83.0 & 88.8 & 499.6& 79.6&  96.2 & 98.5&  63.2&  90.7&  96.1 & 524.3 \\
 NCR&77.3 & 94.0 & 97.5 & 59.6 & 84.4 & 89.9 & 502.7& 78.7&  95.8 & 98.5 & 63.3&  90.4&  95.8&  522.5\\
 DECL&79.8&  94.9&  97.4 & 59.5 & 83.9 & 89.5 & 505.0& 79.1&  96.3 & 98.7&  63.3&  90.1 & 95.6&  523.1\\
 CGMN& 77.9 & 93.8 & 96.8 & 59.9 & 85.1 & {90.6} & 504.1& 76.8&  95.4 & 98.3&  63.8 & 90.7&  95.7&  520.7\\
 UARDA&77.8 & 95.0 & 97.6 & 57.8 & 82.9 & 89.2 & 500.3& 77.8 & 95.0 & 97.6 & 57.8 & 82.9 & 89.2 & 500.3\\
 CMCAN&79.5 &{95.6} &97.6 &60.9 &84.3 &89.9 &507.8&78.6&  96.5& \textbf{98.9}&  63.9 & 90.7 & {96.2} & 524.8\\
 NAAF&78.3&94.1&{97.7}&58.9&83.3&89.0&501.3&78.9&96.0& 	98.7& 	63.1& \textbf{91.4}& \textbf{96.5}&524.6\\
 CCR\&CCS&79.3&95.2 &98.0 &59.8 &83.6 &88.8 &504.7&80.2&\textbf{96.8}& 98.7&{64.3}&  90.6&  95.8&  526.4\\
 RCL&79.9 &\textbf{96.1} &{97.8 }&61.1 &{85.4 }&90.3&510.6&80.4&96.4&{98.7}&64.3&90.8&96.0&526.6\\
 BiCro&\textbf{80.7 }&94.3& 97.6 &59.8 &83.8& 89.7 &505.9& 78.3& 95.8 &98.5& 62.7& 90.0& 95.7& 521.0\\\midrule
 \textbf{CRCL}&78.5&95.5&\textbf{98.0}&\textbf{62.3}&\textbf{86.5}&\textbf{91.7}&\textbf{512.5}&\textbf{80.7}&{96.5}&98.6&\textbf{65.1}&91.2&96.1&\textbf{528.2}
 \\\bottomrule
\end{tabular}}
\end{table}

\subsection{More Results under Synthetic Noisy Correspondences}
To fully demonstrate the superiority and generalization of the proposed CRCL, we provide more comparison results under different robustness frameworks, including DECL\footnote{\url{https://github.com/QinYang79/DECL}}~\cite{qin2022deep} and BiCro\footnote{\url{https://github.com/xu5zhao/BiCro}}~\cite{hu2023cross}. In~\Cref{tb1}, except for the results of BiCro under 20\%, 40\%, and 60\%, all other results are reproduced by us. From the results, our CRCL can significantly improve the robustness of existing methods (\eg VSE$\infty$, SAF, and SGR) and outperform other advanced robust frameworks. It is worth noting that CRCL is also stable and superior in high noise, which shows the effectiveness of our CRCL.

\subsection{More Results under Well-annotated Correspondences}
In this section, we supplement the experimental results under well-annotated correspondences for a comprehensive and faithful comparison, including 17 state-of-the-art baselines, namely VSRN (ICCV'19)~\cite{li2019visual}, CVSE (ECCV'20)~\cite{wang2020consensus}, VSE$\infty$ (CVPR'21)~\cite{chen2021learning},  MV-VSE (IJCAI'22)~\cite{li2022multi}; SCAN (ECCV'18)~\cite{lee2018stacked}, CAMP (ICCV'19)~\cite{wang2019camp}, IMRAM (CVPR'20)~\cite{chen2020imram}, GSMN (CVPR'20)~\cite{liu2020graph}, SGRAF (AAAI'21)~\cite{diao2021similarity}, NCR (NeurIPS'21)~\cite{huang2021learning}, DECL (ACM MM'22)~\cite{qin2022deep}, CGMN (TOMM'22)~\cite{cheng2022cross}, URDA (TMM'22)~\cite{li2022image}, CMCAN (AAAI'22)~\cite{zhang2022show}, NAAF (CVPR'22)~\cite{zhang2022negative}, CCR\&CCS (WACV'23)~\cite{chen2023more},  RCL (TPAMI'23)~\cite{hu2023cross}, and BiCro (CVPR'23) ~\cite{yang2023bicro}. From the experimental results in~\Cref{tb2}, our CRCL achieves competitive results, which demonstrates the ability and potential of CRCL to handle well-correspondence scenarios.

\section{Related works\label{a2}}
\subsection{Image-Text Matching}
Image-text matching methods mainly focus on learning latent visual-semantic relevance/similarities as the evidence for cross-modal retrieval \cite{faghri2017vse++,lee2018stacked,diao2021similarity,chen2021learning,zhang2022show,liu2022regularizing,huang2022mack,goel2022cyclip,pan2023fine,fu2023learning}. These approaches could be roughly classified into global- and local-level methods. To be specific, most global-level methods~\cite{faghri2017vse++,li2019visual,chen2021learning,fu2023learning} project images and texts into a shared global space, wherein cross-modal similarities could be computed \cite{faghri2017vse++,chen2021learning}. For example, Faghri et al.~\cite{faghri2017vse++} proposed a triplet ranking loss with hard negatives to learn holistic visual-semantic embeddings for cross-modal retrieval. A Generalized Pooling Operator (GPO)~\cite{chen2021learning} was proposed to adaptively aggregate different features (\eg region-based and grid-based ones) for better common representations. For the local-level methods, most of them desire to learn the latent fine-grained alignments across modalities for more accurate inference of visual-semantic relevance~\cite{lee2018stacked,diao2021similarity,zhang2022show}. Representatively, Lee et al.\cite{lee2018stacked} proposed a Stacked Cross Attention Network model (SCAN) to excavate the full latent alignments by contextualizing the image regions and word tokens for visual-semantic similarity inference. Diao et al.~\cite{diao2021similarity} proposed a Similarity Graph Reasoning and Attention Filtration model (SGRAF) for accurate cross-modal similarity inference by using a graph convolutional neural network for fine-grained alignments and an attention mechanism for representative alignments. Moreover, Zhang et al.~\cite{zhang2022show}  proposed a novel Cross-Modal Confidence-Aware Network to combine the confidence of matched region-word pairs with local semantic similarities for a more accurate visual-semantic relevance measurement. HREM~\cite{fu2023learning} could explicitly capture both fragment-level relations within modality and instance-level relations across different modalities, leading to better retrieval performance. Pan et.al~\cite{pan2023fine} propose a Cross-modal Hard Aligning Network (CHAN) to comprehensively exploit the most relevant region-word pairs and eliminate all other alignments, achieving better retrieval accuracy and efficiency. 
However, the aforementioned methods rely heavily on well-aligned image-text pairs while ignoring the inevitable noisy correspondences in data \cite{huang2021learning,qin2022deep}, which will mislead the cross-modal learning and lead to performance corruption.

\subsection{Learning with Noisy Labels}

Since the lack of well-annotated data in many real-world applications~\cite{ghosh2017robust,qin2022nim,Feng_2023_CVPR,lin2023graph,hu2021learning,qin2022maximum,qin2021semi,huang2021learning}, learning with incomplete/noisy supervision information is becoming more and more popular in recent years. In this section, we briefly review a few families of these methods against noisy labels: \textbf{1) Robust losses} aims to improve the robustness of loss functions to prevent models from overfitting on noisy labels~ \cite{ghosh2017robust,zhang2018generalized,xu2019l_dmi,wang2019symmetric,ma2020normalized,Feng_2023_CVPR}. \textbf{2) Sample selection}~\cite{li2020dividemix,lu2022ensemble} mainly exploits the memorization effect of DNNs \cite{arpit2017closer} to divide/select the corrupted samples from datasets, and then conduct different training strategies for clean and noisy data. \textbf{3) Correction Approaches}\cite{patrini2017making,tanaka2018joint,lu2022selc} attempt to correct the wrong supervision information (\eg labels or losses) for robust training through some ingenious mechanisms. Different from the aforementioned unimodal category-based methods, learning with noisy correspondence focuses on the noisy annotations existing across different modalities instead of classes \cite{huang2021learning,qin2022deep}. That is to say, noisy correspondences are instance-level noise instead of class-level noise, which is more challenging \cite{huang2021learning,qin2022deep}. To tackle this challenge, Huang et al.~\cite{huang2021learning} first proposed a novel Noisy Correspondence Rectifier (NCR) to rectify the noisy correspondences with co-teaching. By introducing evidential deep learning into image-text matching, Qin et al.~\cite{qin2022deep} proposed a general Deep Evidential Cross-modal Learning framework (DECL) to improve the robustness against noisy correspondences.  Some recent works~\cite{hu2023cross,yang2023bicro} try to predict correspondence labels to recast the margin of triplet ranking loss~\cite{faghri2017vse++} as a soft margin to further improve robustness like NCR, \eg cross-modal mete learning~\cite{hu2023cross} and similarity-based consistency learning~\cite{yang2023bicro}. In addition to image-text matching, other fields are also troubled by NC, such as partially view-aligned clustering~\cite{yang2022robust,yang2021partially,wen2023deep,yang2020adversarial}, video-text retrieval~\cite{zhang2023robust}, visible-infrared person re-identification~\cite{yang2022learning}. In this paper, we mainly focus on the NC problem in image-text matching and try to address this from both robust loss function and correspondence correction.
\end{document}